\documentclass[twoside]{article}
\usepackage[accepted]{aistats2024}

\usepackage[numbers]{natbib}

\usepackage{microtype}
\usepackage{graphicx}
\usepackage{subfigure}
\usepackage{booktabs} 
\usepackage{enumitem}
\usepackage{algorithm, algorithmic}
\usepackage[dvipsnames]{xcolor}
\usepackage[colorlinks,citecolor=NavyBlue,linkcolor=Maroon,urlcolor=blue,breaklinks]{hyperref}

\usepackage{bm}
\usepackage{amsmath}
\usepackage{amssymb}
\usepackage{mathtools}
\usepackage{amsthm}
\usepackage{thm-restate}
\usepackage[capitalize,noabbrev]{cleveref}
\usepackage[textsize=tiny]{todonotes}

\theoremstyle{plain}
\newtheorem{theorem}{Theorem}
\newtheorem{proposition}[theorem]{Proposition}
\newtheorem{lemma}[theorem]{Lemma}

\newtheorem{corollary}[theorem]{Corollary}
\theoremstyle{definition}

\theoremstyle{remark}

\def\b{\bm{b}}
\def\c{\bm{c}}
\def\d{\bm{d}}

\def\f{\bm{f}}
\def\g{\bm{g}}

\def\p{\bm{p}}
\def\r{\bm{r}}

\def\u{\bm{u}}
\def\v{\bm{v}}
\def\w{\bm{w}}
\def\x{\bm{x}}
\def\y{\bm{y}}
\def\z{\bm{z}}

\def\K{\bm{K}}

\def\W{\bm{W}}

\def\alphav{\bm{\alpha}}
\def\betav{\bm{\beta}}
\def\muv{\bm{\mu}}
\def\deltav{\bm{\delta}}

\def\lambdav{\bm{\lambda}}

\def\zeros{\mathbf{0}}

\newcommand{\RR}{\mathbb{R}}

\newcommand{\prox}{\operatorname{prox}}

\DeclareMathOperator*{\argmax}{argmax}
\DeclareMathOperator*{\argmin}{argmin}

\newcommand{\diag}{\operatorname{\bf diag}}
\newcommand{\idm}{\operatorname{I}}
\newcommand{\ones}{\operatorname{\mathbf 1}}

\newcommand{\E}{\mathbb{E}}

\newcommand{\obj}{h}

\newcommand{\softmax}{\operatorname{softmax}}

\begin{document}

\twocolumn[

\aistatstitle{Dual Gauss-Newton Directions for Deep Learning}

\aistatsauthor{Vincent Roulet \And Mathieu Blondel}

\aistatsaddress{Google DeepMind}
]

\begin{abstract}
  Gauss-Newton (a.k.a.\ prox-linear) directions can be computed by solving an
  optimization subproblem that trade-offs between a partial linearization of the
  objective function and a proximity term. In this paper, we study for the first
  time the possibility to leverage the convexity of this subproblem in order to
  instead solve the corresponding dual. As we show, the dual can be advantageous
  when the number of network outputs is smaller than the number of network
  parameters. We propose a conjugate gradient algorithm to solve the dual, that
  integrates seamlessly with autodiff through the use of linear operators and
  handles dual constraints. We prove that this algorithm produces descent
  directions, when run for any number of steps. Finally, we
  study empirically the advantages and current limitations of our approach
  compared to various popular deep learning solvers.
\end{abstract}

\vspace{-0.4cm}
\section{Introduction}\label{sec:intro}

We consider deep learning objectives of the form
\begin{equation}
\min_{\w \in \RR^p}
\left[\frac{1}{n} \sum_{i=1}^n \obj_i(\w)
= \frac{1}{n} \sum_{i=1}^n \ell_i(f_i(\w))\right],
\label{eq:objective}
\end{equation}
where 
$\obj_i \coloneqq \ell_i \circ f_i$, 
$\ell_i \colon \RR^k \to \RR$ is a convex loss on a sample $i \in [n]$,
and $f_i \colon \RR^p \to \RR^k$ is a neural network
with parameters $\w \in \RR^p$, applied on the sample $i$.
Such objectives are generally tackled by algorithms ranging from
stochastic gradient descent to adaptive methods, incorporating
momentum~\citep{sutskever2013importance, kingma2015adam} or linesearch
~\citep{vaswani2019painless}. Stochastic gradients can be computed efficiently
using autodiff frameworks such as JAX~\citep{bradbury2018jax} or
Pytorch~\citep{paszke2017automatic}. However, the information they provide on
the objective function is intrinsically limited, since they amount to using 
a \textbf{full linearization} of $h_i$.

In this paper, we consider constructing update directions through a
\textbf{partial linearization} of the function $h_i$. Such an approach has a
long history starting from the Gauss-Newton and Levenberg-Marquardt
methods~\citep{levenberg1944method, marquardt1963algorithm,
bjorck1996numerical,kelley1999iterative}, used for nonlinear least squares. The
extension to arbitrary convex loss functions of this approach has been called
\textbf{modified Gauss-Newton}~\citep{burke1985descent, nesterov2007modified,
tran2020stochastic,zhang2021stochastic} or \textbf{prox-linear}
\citep{drusvyatskiy2019efficiency, pillutla2019smoother}, name that we will use
in the rest of this paper. In the context of deep learning, several authors have
considered tackling Gauss-Newton-like oracles by decomposing the resulting
problem through the layers of a deep network~\citep{yu2018levenberg,
martens2015optimizing,botev2017practical, gupta2018shampoo, anil2020scalable},
see Appendix~\ref{app:related_work} for a detailed literature review.

In this paper, we propose instead to exploit the \textbf{convexity} of the
subproblem associated with such prox-linear directions
by switching to their \textbf{dual formulation}.
We make the following contributions.
\begin{itemize}[nosep, leftmargin=10pt]

\item After reviewing prox-linear (a.k.a. modified Gauss-Newton) directions,
computed by solving an
optimization subproblem that trade-offs between a partial linearization of the
objective function and a proximity term,
we cleanly derive the subproblem's {\bf Fenchel-Rockafellar dual}. 

\item To solve this dual, we present a {\bf conjugate gradient} algorithm.
Our proposal integrates seamlessly in an autodiff framework, by leveraging
Jacobian-vector products (JVP) and vector-Jacobian products (VJP)
and can handle dual constraints.

\item We prove that the proposed algorithm produces a {\bf descent
      direction}, when run for any number of steps, enabling its use as a
      drop-in replacement for the stochastic gradient in existing
      algorithms, such as SGD.

\item Finally, we present comprehensive empirical results demonstrating the 
    \textbf{advantages} and \textbf{current limitations} of our approach.

\end{itemize}

\vspace{-0.4cm}
\section{Prox-linear directions via the primal}\label{sec:proxlin}

In this section, we review prox-linear directions, which are based on the idea of 
\textbf{partial linearization}.

\subsection{Variational perspective}

To motivate prox-linear directions, we review the
variational perspective on gradient and Newton directions.
Let the \textbf{linear approximation} of $h_i$ around $\w^t$ be
\begin{equation*}
\mathrm{lin}(\obj_i, \w^t)(\w) 
\coloneqq 
\obj_i(\w^t) 
+ \langle \nabla \obj_i(\w^t), \w - \w^t \rangle
\approx \obj_i(\w).
\end{equation*}
From a variational perspective, the \textbf{stochastic gradient} can be seen as
the minimization of this linear approximation and of a quadratic
regularization term,
\begin{equation*}
\gamma \nabla \obj_i(\w^t)
= \argmin_{\d \in \RR^p} 
\mathrm{lin}(\obj_i, \w^t)(\w^t - \d) + \frac{1}{2\gamma} \|\d\|^2_2.
\end{equation*}
The information provided by a gradient is therefore limited by the quality of a
linear approximation of the objective. 
Alternatively, we may consider the \textbf{quadratic approximation} of $h_i$
around $\w^t$,
\begin{align*}
\mathrm{quad}(h_i, \w^t)(\w) 
\coloneqq&
h_i(\w^t) + \langle \nabla h_i(\w^t), \w - \w^t \rangle + \\
&\frac{1}{2} \langle \w - \w^t, \nabla^2 h_i(\w^t) (\w - \w^t) \rangle.
\end{align*}
The \textbf{regularized Newton direction} is then
\begin{align*}
&\argmin_{\d \in \RR^p} 
\mathrm{quad}(\obj_i, \w^t)(\w^t - \d) + \frac{1}{2\gamma} \|\d\|^2_2 \\
=&
(\nabla^2 \obj_i(\w^t) + \gamma^{-1} \idm)^{-1} \nabla \obj_i (\w^t).
\end{align*}
Unfortunately, when $f_i$ is a neural network, $h_i = \ell_i \circ f_i$ 
is typically nonconvex and
the Hessian $\nabla^2 h_i(\w)$ is an \textbf{indefinite} matrix. 
Therefore, the minimization above is that of a \textbf{nonconvex
quadratic} subproblem, that may be hard to solve. Furthermore, the obtained
direction is \textbf{not} guaranteed to define a \textbf{descent direction}.

\subsection{Convex-linear approximations}

Instead of \textbf{fully linearizing} $\obj_i$, which amounts to linearizing
\textbf{both} $\ell_i$ and $f_i$, we may linearize \textbf{only} $f_i$ but keep
$\ell_i$ as is. That is, we may use the \textbf{partial linear approximation} of
$\obj_i$,
\begin{equation*}
\begin{aligned}
\operatorname{plin}(\ell_i, f_i, \w^t)(\w)
&\coloneqq \ell_i(f_i(\w^t) + \partial f_i(\w^t)(\w - \w^t)) \\
&= \ell_i(\f_i^t + J_i^t (\w - \w^t)),
\end{aligned}
\end{equation*}
where we defined the shorthands
\begin{align*}
\f_i^t &\coloneqq f_i(\w^t)
\quad \text{and} \quad
J_i^t \coloneqq \partial f_i(\w^t).
\end{align*}
We say that the approximation is \textbf{convex-linear}, since it is the
composition of the convex $\ell_i$ and of the linear approximation of $f_i$.
We view $J_i^t$ as a \textbf{linear map} from $\RR^p$ to $\RR^k$, 
with \textbf{adjoint operator} $(J_i^t)^*$, a linear map from $\RR^k$ to
$\RR^p$.
That is, we assume that we can compute the Jacobian-vector product (JVP)
$(J_i^t) \u$ for any direction $\u \in \RR^p$ and the vector-Jacobian product
(VJP) $(J_i^t)^* \v_i$ for any direction $\v_i \in \RR^k$.
This will be the case if $f_i$ is implemented using an autodiff framework such
as JAX or PyTorch.

The minimization of the convex-linear approximation and of a quadratic
regularization term leads to the definition of the \textbf{prox-linear}
(a.k.a. \textbf{modified Gauss-Newton}) \textbf{direction}:
\begin{align}\label{eq:subproblem_i}
& \d(\gamma \ell_i, f_i)(\w^t)\nonumber\\
&  \coloneqq \argmin_{\d \in \RR^p} 
\mathrm{plin}(\ell_i, f_i, \w^t)(\w^t - \d)
+ \frac{1}{2\gamma}\|\d\|_2^2 \nonumber\\
& = \argmin_{\d \in \RR^p}  \ell_i(\f_i^t - J_i^t \d) 
+ \frac{1}{2\gamma}\|\d\|_2^2.
\end{align}
Typically, we will only obtain an \textbf{approximate solution},
that we denote $\d_i^t \approx \d(\gamma \ell_i, f_i)(\w^t)$.
Once we obtained $\d_i^t$, we can update the parameters as
\begin{equation*}
\w^{t+1} \coloneqq \w^t - \d_i^t. 
\end{equation*}
We emphasize, however, that unlike the gradient, the prox-linear direction is
not linear in $\gamma$, i.e.,
\begin{equation*}
\w^t - \d(\gamma \ell_i, f_i)(\w^t)
\neq
\w^t - \gamma \d(\ell_i, f_i)(\w^t).
\end{equation*}

\paragraph{Mini-batch extension.}

The prox-linear direction presented earlier is obtained from a single sample $i
\in [n]$. 
We now consider the extension to a mini-batch 
$S \coloneqq \{i_1, \dots, i_m\} \subseteq [n]$ of size $m = |S|$.
Let us define
\begin{align*}
f_S(\w) 
&\coloneqq (f_{i_1}(\w), \ldots, f_{i_m}(\w))^\top \in \RR^{m \times k}\\
\ell_S(\f_S) 
&\coloneqq \sum_{j=1}^m \ell_{i_j}(\f_i) \in \RR.
\end{align*}
We then define the mini-batch direction
\begin{align}\label{eq:subproblem_minibatch}
&\d(\gamma \ell_S, f_S)(\w^t) \nonumber \\
\coloneqq&  \argmin_{\d \in \RR^p} 
\frac{1}{m}\sum_{i \in S} \mathrm{plin}(\ell_i, f_i)(\w^t)
+ \frac{1}{2\gamma}\|\d\|_2^2 \nonumber \\
=&  \argmin_{\d \in \RR^p} 
 \sum_{i \in S} \ell_i(\f_i^t -J_i^t \d) 
+ \frac{m}{2\gamma}\|\d\|_2^2.
\end{align}
Typically, we only compute an approximate solution,
denoted $\d_S^t \approx \d(\gamma \ell_S, f_S)(\w^t)$.
The above subproblem is again convex.
It can be solved
directly in the primal, or, as we propose in Section \ref{sec:dual}, 
by switching to the dual.
Once we obtained $\d_S^t$, we may perform the update
\begin{equation*}
\w^{t+1} \coloneqq \w^t - \d^t_S,
\end{equation*}
or use the direction $\d_S^t$ in an existing stochastic algorithm.
Unlike the mini-batch stochastic gradient, the mini-batch prox-linear direction
is not the average of the individual stochastic directions, i.e., 
\begin{equation*}
\d(\ell_S, f_S)(\w^t) 
\neq 
\frac{1}{m} \sum_{i \in S} \d(\ell_i, f_i)(\w^t).
\end{equation*}
The mini-batch prox-linear direction can therefore take advantage of the
correlations between the samples of the mini-batch, unlike the gradient oracle.

\subsection{Quadratic-linear approximations}

In the previous section, we studied convex-linear approximations, i.e.,
the composition of a convex $\ell_S$ and of the linear approximation of $f_S$,
on a sample $S$.
As in the so-called \textbf{``generalized'' Gauss-Newton}
algorithm~\citep{gargiani2020promise}, we can replace $\ell_S$ with its
quadratic approximation around $f_S(\w^t)$, namely,
\begin{equation*}
q_S^t \coloneqq \mathrm{quad}(\ell_S, f_S(\w^t)).
\end{equation*}
Replacing $\ell_S$ with $q_S^t$,
we get after simple algebraic manipulations (see Appendix~\ref{app:proofs})
a \textbf{convex quadratic}
\begin{align}\label{eq:quadratic_subproblem_minibatch} 
&\d(\gamma q^t_S, f_S)(\w^t) \nonumber \\
=& \argmin_{\d \in \RR^p} 
\frac{1}{m} \sum_{i \in S} q_i^t(\f_i^t - J_i^t \d) + \frac{1}{2\gamma}
\|\d\|^2_2 \\
=&
\argmin_{\d \in \RR^p} \frac{1}{2}
\langle  \d, (J_S^t)^* H_S^t J_S^t \d \rangle  
- \langle  (J_S^t)^* \g_S^t, \d \rangle + \frac{m}{2\gamma}\|\d\|_2^2 \nonumber \\
\approx& \d(\gamma \ell_S, f_S)(\w^t) \nonumber,
\end{align}
where, denoting $\f_S^t \coloneqq f_S(\w^t)$, we defined
\begin{small}
\begin{align*}
J_S^t \u 
&\coloneqq \partial \f_S(\w^t)\u   
= (J_{i_1}^t \u, \ldots, J_{i_m}^t \u)^\top \\
(J_S^t)^*\v 
&\coloneqq \partial \f_S(\w^t)^* \v  
= \sum_{j=1}^m (J_{i_j}^t)^* \v_j \\
H_S^t \v 
&\coloneqq \nabla \ell_S^2(\f_S^t)\v
=  \left(\nabla^2 \ell_{i_1}(\f_1^t)\v_1, \ldots, \nabla^2
\ell_{i_m}(\f_m^t)\v_m\right) \\
\g_S^t
&\coloneqq \nabla \ell_S (\f_S^t) 
=(\nabla \ell_{i_1}(\f_1^t), \ldots, \nabla \ell_{i_m}(\f_m^t)).
\end{align*}
\end{small}
\hspace*{-7pt}
In constrast, we recall that the subproblem associated with the Newton direction
is a \textbf{nonconvex quadratic}.

Without a regularization term, that is by using $\gamma=+\infty$
in~\eqref{eq:quadratic_subproblem_minibatch}, $\d(q^t_S, f_S)$ amounts to a
generalized Gauss-Newton step~\citep{gargiani2020promise}, which itself matches
a natural gradient step~\citep{amari1998natural} if $\ell_i$ is the negative
log-likelihood of a regular exponential family~\citep{martens2020new}.

\paragraph{Examples.}

For the \textbf{squared loss}, we have 
\begin{align*}
    \ell_i(\f_i) & \coloneqq \frac{1}{2} \|\f_i - \y_i\|^2_2, \
    \nabla \ell_i(\f_i) = \f_i - \y_i, \\
    \nabla^2 \ell_i(\f_i) & = \idm, \hspace*{41pt}
    \nabla^2 \ell_i(\f_i) \v_i = \v_i.
\end{align*}
For the \textbf{logistic loss}, we have
\begin{align*}
    \ell_i(\f_i) & \coloneqq 
\mathrm{LSE}(\f_i) - \langle \f_i, \y_i \rangle, \ 
\nabla \ell_i(\f_i) = \sigma(\f_i) - \y_i, \\
\nabla^2\ell_i(\f_i) & = \diag(\sigma(\f_i)) -
\sigma(\f_i)\sigma(\f_i)^\top, \\
\nabla^2\ell_i(\f_i)\v_i & = \sigma(\f_i) \odot \v_i -
\langle \v_i, \sigma(\f_i)\rangle \sigma(\f_i),
\end{align*}
where 
$\sigma(\f_i) \coloneqq \softmax(\f_i)$.

\subsection{Practical implementation}

\paragraph{Computing an approximate direction.}

As we saw, computing a direction involves the (approximate) resolution of the
convex subproblem \eqref{eq:subproblem_minibatch} or of the convex quadratic
subproblem \eqref{eq:quadratic_subproblem_minibatch}. In general, we must resort
to an iterative convex optimization algorithm.

On one extreme, performing a \textbf{single} gradient step on the convex
subproblem would bring no advantage, since it would be equivalent
to a gradient step on the nonconvex original problem. Indeed, we have
\begin{equation*}
\nabla[\mathrm{plin}(\ell_i, f_i, \w^t)](\w^t)
= \nabla h_i(\w^t)
= (J_i^t)^* \nabla \ell_i(f_i(\w^t)),
\end{equation*}
and generally
$
\nabla[\mathrm{plin}(\ell_S, f_S, \w^t)](\w^t)
= \nabla h_S(\w^t)$.

On the other extreme, if we solve the subproblem to high accuracy, the overhead
of solving the subproblem may hinder the benefit of a better direction than
the gradient.
A trade-off must be struck between the additional computational
complexity of the subproblem and the benefits of a refined
direction ~\citep{lin2018catalyst,drusvyatskiy2019efficiency}. 

We argue that a good inner solver should meet the following requirements.
\begin{enumerate}[topsep=0pt,itemsep=2pt,parsep=2pt,leftmargin=10pt]
\item It should be easy to implement or widely available.
\item It should be compatible with linear maps, i.e., it should not require to
materialize $J_S^t$ as a matrix in memory. Such an algorithm is called
\textbf{matrix-free}.
\item It should not require to tune hyperparameters.
\item It should leverage the specificities of the subproblem.
\end{enumerate}

If we decide to solve the primal instead, as was done in the existing
literature, LBFGS~\citep{nocedal1999numerical} is a good generic candidate, but
that does not especially leverage the nature of the subproblem.
In the full-batch setting, where $m=n$, 
\citet{drusvyatskiy2019efficiency,pillutla2023modified} considered
variance-reduced algorithms such as SVRG.
In the case of the convex quadratic \eqref{eq:quadratic_subproblem_minibatch},
we can use the conjugate gradient method.

Whatever the inner algorithm used, we can control the trade-off between
computational cost
and precision using a maximum number of inner iterations.

\paragraph{Performing an update.}

Once we obtained an approximate direction
$\d_S^t \approx \d(\gamma \ell_S, f_S)(\w^t)$,
we already saw that we can simply perform the update
\begin{equation}
\w^{t+1} \coloneqq \w^t - \d_S^t.
\label{eq:minibatch_update}
\end{equation}
This may require the tuning of the regularization parameter $\gamma$, which
effectively act as a stepsize by analogy with the variational formulation of the
gradient.

More generally, we can use $\d_S^t$ as a \textbf{drop-in replacement} for the
stochastic gradient direction in other algorithms such as SGD with momentum or
SGD with linesearch.
For example, we may fix the regularization parameter to some value, say
$\gamma = 1$, and instead perform the update
\begin{equation}
\w^{t+1} \coloneqq \w^t - \eta^t \d_S^t,
\label{eq:minibatch_linesearch_update}
\end{equation}
where $\eta^t$ is a stepsize (we use a different letter to distinguish it from
the regularization parameter $\gamma$), selected by
a backtracking Armijo linesearch. That is,
we seek $\eta^t$ satisfying
\begin{equation*}
h_S(\w^t - \eta^t \d_S^t)
\le
h_S(\w^t) - \beta \eta^t \langle \d_S^t, \g_S^t \rangle,
\end{equation*}
where $h_S = \frac{1}{m} \ell_S\circ f_S(\w)$, $\g_S^t \coloneqq \nabla
h_S(\w^t)= \frac{1}{m} \sum_{i \in S} \nabla h_i(\w^t)$ is the mini-batch
stochastic gradient evaluated at $\w^t$, and where $\beta \in (0,1)$ is a
standard constant, typically set to $\beta =10^{-4}$.
The entire procedure is summarized in Algorithm \ref{algo:proxlin_template}.

\begin{algorithm}[t]
\caption{Primal-based prox-linear direction\label{algo:proxlin_template}}
  \begin{algorithmic}[1]
    \STATE{\bf Inputs:} parameters $\w^t$, batch $S =\{i_1, \ldots, i_m\}
    \subseteq [n]$, regularization $\gamma>0$ (set to $1$ if linesearch is used)
    \STATE Compute
    $
    \f_S^t = (f_{i_1}(\w^t), \dots, f_{i_m}(\w^t))^\top \in \RR^{m \times k}
    $
    \STATE Instantiate JVP and VJP operators by autodiff:
    \begin{align*}
        &\u \mapsto J_S^t \u = (J_{i_1}^t \u, \dots, J_{i_m}^t \u) \in
        \RR^{m \times k},
        \quad \forall \u \in \RR^p \\
        &\v \mapsto (J_S^t)^* \v = \sum_{j=1}^m (J_{i_j}^t)^* \v_j \in \RR^p,
        \quad \forall \v \in \RR^{m \times k}
    \end{align*}
    \vspace*{-10pt}
    \STATE Run inner solver to approximately solve
    \eqref{eq:subproblem_minibatch}, i.e.,
    \[
        \d_S^t \approx \argmin_{\d \in \RR^p} \frac{1}{m} \sum_{i \in S}
      \ell_i(\f_i^t -J_i^t \d) + \frac{1}{2\gamma }\|\d\|_2^2,
    \]
    \vspace*{-5pt}
    or its quadratic approximation \eqref{eq:quadratic_subproblem_minibatch},
    i.e., \label{line:descent_dir_primal}
    \begin{align*}
        \d_S^t &\approx \argmin_{\d \in \RR^p} \frac{1}{m} \sum_{i \in S}
      q_i^t(\f_i^t -J_i^t \d) + \frac{1}{2 \gamma}\|\d\|_2^2
    \end{align*}

\STATE Set next parameters $\w^{t+1}$ by 
\begin{align*}
  \w^{t+1} & \coloneqq \w^t - \d_S^t 
  \hspace*{30pt} \mathrm{(fixed~stepsize~\eqref{eq:minibatch_update})} \\
  \mbox{or} \quad \w^{t+1} & \coloneqq \w^t - \eta^t \d_S^t 
  \hspace*{35.5pt} \mathrm{(linesearch~\eqref{eq:minibatch_linesearch_update})}
\end{align*}
for $\eta^t$ s.t.  $h_S(\w^t - \eta^t \d_S^t)
\le
h_S(\w^t) - \beta \eta^t \langle \d_S^t, \g_S^t \rangle.$
\STATE{\bf Output:} $\w^{t+1} \in \RR^p$
  \end{algorithmic}
\end{algorithm}

\section{Prox-linear directions via the dual}\label{sec:dual}

In this section, we study how to obtain an approximate prox-linear direction,
by solving the convex subproblem
\eqref{eq:subproblem_minibatch} or the convex quadratic subproblem
\eqref{eq:quadratic_subproblem_minibatch} in the dual.

\subsection{Convex-linear approximations}

By taking advantage of the availibility of the conjugate
\begin{equation*}
\ell^*_i(\alphav_i) 
\coloneqq \sup_{\f_i \in \RR^k} \langle \f_i, \alphav_i \rangle -
\ell_i(\f_i),
\end{equation*}
we can express the prox-linear direction~\eqref{eq:subproblem_i} on a single
sample $i \in [n]$ as 
\begin{equation*}
\d(\gamma\ell_i, f_i)(\w^t) = \gamma (J_i^t)^* \alphav(\gamma \ell_i, f_i)(\w^t),
\end{equation*}
where we defined the solution of the dual of \eqref{eq:subproblem_i},
\begin{equation*}
\alphav(\gamma \ell_i, f_i)(\w^t) \coloneqq
\argmin_{\alphav_i \in \RR^k} 
\ell_i^*(\alphav_i) 
- \langle \alphav_i, \f_i^t \rangle
+ \frac{\gamma}{2} \|(J_i^t)^*\alphav_i\|_2^2.
\end{equation*}
Let us compare this with a stochastic gradient:
\begin{equation*}
\gamma \nabla h_i(\w^t) = \gamma (J_i^t)^* \nabla \ell_i(\f_i^t).
\end{equation*}
The dual viewpoint reveals that
the prox-linear direction can be seen as replacing the gradient of
the loss $\nabla \ell_i(\f_i^t)$ by the solution of the subproblem's dual
$\alphav(\gamma \ell_i, f_i)(\w^t)$. This also suggests that
$\nabla \ell_i(\f_i^t)$ is a \textbf{good initialization} for $\alphav_i$.

\paragraph{Benefit of using the dual.}

The dual subproblem involves $k$ variables while the primal subproblem involves
$p$ variables. Typically, $k$ is the number of network outputs (e.g., classes),
while $p$ is the number of network parameters. The dual subproblem is therefore
advantageous when $k \ll p$, which is often the case.

\paragraph{Mini-batch extension.}

For the mini-batch case, 
the prox-linear direction can be computed as
\begin{equation*}
\d(\gamma \ell_S, f_S) (\w^t)
= \frac{\gamma}{m} (J_S^t)^* \alphav(\gamma \ell_S, f_S)(\w^t),
\end{equation*}
where we defined the dual solution of \eqref{eq:subproblem_minibatch} by
\begin{align}
&\alphav(\gamma \ell_S, f_S)(\w^t) \coloneqq\nonumber \\
&\argmin_{\alphav_S \in \RR^{m \times k}} 
\ell_S^*(\alphav_S) 
- \langle \alphav_S, \f_S^t \rangle
+ \frac{\gamma}{2m} \|(J_S^t)^* \alphav_S\|_2^2,
\label{eq:minibatch_dual_subproblem}
\end{align}
where $\ell_S^*(\alphav_S) \coloneqq \sum_{i \in S} \ell_i^*(\alphav_i)$.

This time, the dual subproblem involves $m \times k$ variables, 
where $m$ is the mini-batch size and $k$ is the number of network outputs,
while the primal subproblem involves $p$ variables, as before. 
If the mini-batch is not too large, we typically have
$mk \ll p$ and therefore the dual subproblem is still advantageous.
Algorithm~\ref{algo:dual_proxlin} summarizes our approach.

\paragraph{Examples.}

For the \textbf{squared loss}, the conjugate is
\[
\ell_i^*(\alphav_i) =
\frac{1}{2}\|\alphav_i\|_2^2 + \langle \alphav_i, \y_i \rangle.
\]
The dual subproblem therefore becomes
\begin{align*}
&\alphav(\gamma \ell_S, f_S)(\w^t) = \\
&\argmin_{\alphav_S \in \RR^{m \times k}} 
\sum_{i=1}^m \frac{1}{2} \|\alphav_i\|^2_2 
- \langle \alphav_i, \g_i^t \rangle
+ \frac{\gamma}{2m} \|(J_S^t)^* \alphav_S\|^2_2, 
\end{align*}
where $\g_i^t \coloneqq \f_i^t - \y_i = \nabla \ell_i(\f_i^t)$.
Setting the gradient to zero, this gives the linear system
\begin{equation*}
\left(\frac{\gamma}{m} J_S^t (J_S^t)^* + \idm\right) \alphav_S^t =\g_S^t, 
\end{equation*}
where $\g_S^t \coloneqq (\g_i^t)_{i \in S} \in \RR^{m\times k}$ and $\alphav_S^t
= \alphav(\gamma \ell_S, f_S)(\w^t)$. We can solve this system using the
conjugate gradient method. Note that if $\alphav_i^t$ is equal to the residual
$\y_i - \f_i^t$, then $(J_i^t)^* \alphav_i^t = \nabla h_i(\w^t)$. Therefore, the
residual is a good initialization for the conjugate gradient method.

For the \textbf{logistic loss}, the conjugate is
\begin{equation*}
\ell_i^*(\alphav_i) = \langle \muv_i,  \log \muv_i \rangle +
\iota_{\triangle^k}(\muv_i), \quad \mbox{for} \ \muv_i \coloneqq \y_i + \alphav_i,
\end{equation*}
where $\iota_{\triangle^k}$ is the indicator function of the probability simplex
$\triangle^k \coloneqq \{\muv \in \RR^k \colon \muv \geq 0, \muv^\top \ones =
\ones\}$.
 
Applying this conjugate, we arrive at the subproblem
\begin{align*}
\muv_S^t\coloneqq 
\argmin_{\substack{\muv_S \in \RR^{m\times k}\\ \muv_i \in \triangle^k}} 
& 
\langle \muv_S, \log \muv_S \rangle
- \langle \muv_S -\y_S, \f_S^t \rangle \\
&  
+ \frac{\gamma}{2m} \|(J_S^t)^* (\muv_S- \y_S)\|^2_2,
\end{align*}
where $\y_S \coloneqq (\y_{i_1}, \dots, \y_{i_m})^\top \in \RR^{m\times k}$.
This is a \textbf{constrained} convex optimization problem,
that can be solved by, e.g., projected gradient descent. 
Changing the variable back, we obtain
\begin{equation*}
\alphav(\gamma \ell_S, f_S)(\w^t) = \muv_S^t - \y_S. 
\end{equation*}

\subsection{Quadratic-linear approximations}

To enable the use of the conjugate gradient method, we consider the dual of the
quadratic approximation presented earlier. 
If the Hessian of the loss $\ell_i$ is invertible, the dual solution 
$\alphav(\gamma q_S^t, f_S)(\w^t)$ 
of
\eqref{eq:quadratic_subproblem_minibatch} equals
\begin{align*}
\argmin_{\alphav_S \in \RR^{m \times k}} 
(q_S^t)^*(\alphav_S) - \langle \alphav_S, \f_S^t\rangle 
+ \frac{\gamma}{2m} \|(J_S^t)^* \alphav_S\|_2^2,
\end{align*}
for $(q_S^t)^*(\alphav_S) - \langle \alphav_S, \f_S^t\rangle = 
\langle \g_S^t - \alphav_S, H_S^{-1}
(\g_S^t - \alphav_S) \rangle, \nonumber$
where we used the inverse-Hessian-vector product
\begin{equation*}
H_S^{-1} \alphav_S \coloneqq (\nabla \ell^2_{i_1}(\f_{i_1}^t)^{-1}\alphav_1,
\ldots, \nabla \ell^2_{i_m}(\f_{i_m}^t)^{-1}\alphav_m).
\end{equation*}
For the \textbf{squared loss}, we naturally get back the solution presented
earlier. For the \textbf{logistic loss}, while the Hessian is positive
semi-definite, it is not invertible, as we have $\nabla^2 \ell_i(\f_i^t)^\top
\ones_k = 0$ for all $i \in [n]$. 

Generally for any positive-semi-definite Hessian, the dual can still be
formulated as an \textbf{equality-constrained QP}, see
Appendix~\ref{app:dual_quad_logistic}. The direction can be computed as $\d_S^t
\approx \frac{\gamma}{m}(J_S^t)^* \alphav_S^t$, with $\alphav_S^t = \g_S^t -
\betav_S^t$  and
\begin{align}
\betav_S^t \approx
  \argmin_{\substack{\betav \in \RR^{m\times k}}} & 
  \frac{1}{2} \langle \betav ,
  (H_S^t)^{\dagger} \betav \rangle
  {+} \frac{\gamma}{2m} \|(J_S^t)^* (\g_S^t - \betav)\|_2^2 \nonumber\\
  \mbox{s.t.} \ & (\idm - H_S^t(H_S^t)^{\dagger})\betav ={\bm 0}.
\label{eq:dual_quadratic_approx}
\end{align}
where we used the pseudo-inverse Hessian product
\begin{equation}\label{eq:pseudo_inverse}
  (H_S^t)^{\dagger} \alphav \coloneqq ((H_{i_1}^t)^{\dagger} \alphav_1, \ldots,
(H_{i_m}^t)^\dagger \alphav_m)
\end{equation}
The above
equality-constrained QP can be solved efficiently with a \textbf{conjugate
gradient} method, using $H_S^t(H_S^t)^{\dagger}$ as a preconditioner and
initializing at a dual variable respecting the constraints, see
Appendix~\ref{app:dual_quad_logistic}. 
Note that if the
subproblem in $\betav$ is initialized at ${\bm 0}$ as done in the experiments,
the output direction is a gradient at iteration 0. Each subsequent iteration
therefore improves on the gradient.

\paragraph{Example.} 

For the {\bf logistic loss}, 
the pseudo-inverse enjoys a \textbf{closed form}. 
The direction can then be computed as $\d_S^t
\approx \frac{\gamma}{m}(J_S^t)^* \alphav_S^t$, with $\alphav_S^t = \g_S^t - \betav_S^t$  for
\begin{align}
\betav_S^t \approx
\argmin_{\substack{\betav \in \RR^{m \times k}}} & 
\frac{1}{2}\langle \betav, D_S^{-1} \betav \rangle +
\frac{\gamma}{2m} \|(J_S^t)^*(\g_S^t -\betav)\|_2^2 \nonumber\\
\mbox{s.t.} \ & \ones_k^\top \betav_i = {\bm 0}, \ i \in [m].
\label{eq:dual_quad_logistic}
\end{align}
Here, 
denoting $\sigma$ the softmax, we defined
\begin{equation*}
D_S^{-1} \betav
\coloneqq (\betav_1 / \sigma(\f_{i_1}^t), \ldots, \betav_m /\sigma(\f_{i_m}^t)),
\end{equation*}
where division is performed element-wise.

\subsection{Linear case: connection with SDCA}

We now discuss the setting when $f_i$ is linear.
In linear multiclass classification,
with $k$ classes and $d$ features,
we set $f_i(\w) = \W \x_i$, where $\W \in \RR^{k \times d}$ is a matrix
representation of $\w \in \RR^p$, with $p=kd$.
We then have
\begin{equation*}
(J_i^t)^* \alphav_i 
= \partial f_i(\w^t)^* \alphav_i 
= \mathrm{vec}(\alphav_i \x_i^\top) \in \RR^{kd}.
\end{equation*}
The key difference with the nonlinear $f_i$ setting is that the Jacobian $J_i^t$
is actually independent of $\w^t$, so that
\begin{equation*}
\|(J_i^t)^* \alphav_i\|^2_2
= \langle \alphav_i, \|\x_i\|_2^2 \idm \alphav_i \rangle
= \|\x_i\|_2^2 \cdot \|\alphav_i\|^2_2.
\end{equation*}
Contrary to the setting where $f_i$ is nonlinear, 
the Hessian of $\|(J_i^t)^* \alphav_i\|^2_2$
is diagonal, which makes the subproblem easier to solve.
Indeed, when the batch size is $m=1$, 
denoting $\sigma_i \coloneqq (\gamma \|\x_i\|_2^2)^{-1}$, we arrive at the dual
subproblem 
\begin{align*}
\alphav(\gamma \ell_i, f_i)(\w^t)
&=
\argmin_{\alphav_i \in \RR^k} 
\ell_i^*(\alphav_i) 
- \langle \alphav_i, \f_i^t \rangle 
+ \frac{1}{2\sigma_i} \|\alphav_i\|_2^2 \\
&= \mathrm{prox}_{\sigma_i \ell_i^*}(\sigma_i \f_i^t).
\end{align*}
This is exactly the dual subproblem used in SDCA \citep{shalev2013stochastic}.
It enjoys a closed form in the case of the squared, hinge and sparsemax loss
functions \citep{blondel2020learning}.
When the batch size is $m$, we obtain that the dual subproblem
solution $\alphav(\gamma \ell_S, f_S)(\w^t)$ is equal to
\begin{equation*}
\argmin_{\alphav_S \in \RR^{m \times k}} 
\ell_S^*(\alphav_S) 
- \langle \alphav_S, \f_S^t \rangle 
+ \frac{\gamma}{2m} \langle \alphav_S, \K \alphav_S \rangle,
\end{equation*}
where we defined the kernel matrix $[\K]_{i,j} \coloneqq \langle \x_i, \x_j
\rangle$. There is no closed form in this case.

\begin{algorithm}[t]
\caption{Dual-based prox-linear direction \label{algo:dual_proxlin}}
\begin{algorithmic}[1]
\STATE{\bf Inputs:} network outputs $\f_S^t$, JVP $J_S^t$ and VJP $(J_S^t)^*$ as
in Algorithm~\ref{algo:proxlin_template}, regularization $\gamma$ ($1$ if
linesearch is used)

\STATE Run inner solver to approximately solve \eqref{eq:minibatch_dual_subproblem}
\begin{equation*}
  \hspace*{-10pt}
\alphav_S^t \approx 
\argmin_{\alphav_S \in \RR^{m \times k}} 
\ell_S^*(\alphav_S) - \langle \alphav_S, \f_S^t \rangle
+ \frac{\gamma}{2m} \|(J_S^t)^*\alphav_S\|_2^2,
\end{equation*}
or its quadratic approximation \label{line:descent_dir_dual}
\begin{align*}
\hspace*{-10pt}
\alphav_S^t \approx 
\argmin_{\alphav_S \in \RR^{m \times k}}
(q_S^t)^*(\alphav_S) {-} \langle \alphav_S, \f_S^t\rangle 
{+} \frac{\gamma}{2m} \|(J_S^t)^* \alphav_S\|_2^2,
\end{align*}
detailed in \eqref{eq:dual_quadratic_approx}, and in \eqref{eq:dual_quad_logistic} for the logistic loss.

\STATE Map back to primal direction \label{line:dual_mapping}
\begin{equation*}
\d^t_S 
\coloneqq \frac{\gamma}{m} (J_S^t)^* \alphav_S^t
= \frac{\gamma}{m} \sum_{i \in S} (J_i^t)^* \alphav_i^t
\end{equation*}

\STATE Set next parameters $\w^{t+1}$ by 
\begin{align*}
  \w^{t+1} & \coloneqq \w^t - \d_S^t 
  \hspace*{30pt} \mathrm{(fixed~stepsize~\eqref{eq:minibatch_update})} \\
  \mbox{or} \quad \w^{t+1} & \coloneqq \w^t - \eta^t \d_S^t 
  \hspace*{35.5pt} \mathrm{(linesearch~\eqref{eq:minibatch_linesearch_update})}
\end{align*}
for $\eta^t$ s.t.  $h_S(\w^t - \eta^t \d_S^t)
\le
h_S(\w^t) - \beta \eta^t \langle \d_S^t, \g_S^t \rangle.$

\STATE{\bf Output:} $\w^{t+1} \in \RR^p$
\end{algorithmic}
\label{algo:dual}
\end{algorithm}

\section{Analysis}\label{sec:analysis}

We now review prox-linear directions theoretically.

\paragraph{Approximation error.}

When $\ell_i$ is $C_\ell$-Lipschitz continuous and $f_i$ is $L_{f}$-smooth, 
the partial linearization of $h_i \coloneqq \ell_i \circ f_i$ satisfies a
quadratic approximation error~\citep[Lemma
3.2]{drusvyatskiy2019efficiency}, for all $\w, \u \in \RR^p$,
\[
|h_i(\w + \u) -  \ell_i(f_i(\w) + \partial f_i(\w)\u)| \leq 
 \frac{C_\ell L_{f}}{2}\|\u\|_2^2.
\]
In comparison, if in addition, $\ell_i$ is $L_\ell$ smooth and $f_i$ is
$C_{f}$-Lipschitz continuous, a linear approximation of $h_i$
has a quadratic error of the form
\[
  |h_i(\w+ \u) - h_i(\w)  - \nabla h_i(\w)\u| 
  \leq \frac{C_\ell L_{f} + C_{f}^2 L_\ell}{2} \|\u\|_2^2.
\]
The above result confirms that, unsurprisingly, the partial linearization is
theoretically a more accurate approximation than the full linearization.

\paragraph{Descent direction.}

To integrate a prox-linear direction $\d=\d(\gamma\ell_S, f_S)$ or $
\d(\gamma q_S^t, f_S)$, into generic optimization algorithms, it is preferable
if $-\d$ defines a \textbf{descent direction} w.r.t. the mini-batch stochastic
gradient $\nabla h_S(\w)$, namely, $-\d$ should satisfy
\begin{equation}
\langle -\d,  \nabla h_S(\w)\rangle \leq 0.
\label{eq:descent_direction}
\end{equation}
Thanks to the convexity of the subproblem, we can show that the exact
prox-linear direction satisfies \eqref{eq:descent_direction}.
\begin{proposition}
    If each $\ell_i$ is convex, the negative
    direction $-\d(\gamma \ell_S, f_S)(\w^t)$ or its
    quadratic approximation direction
    $-\d(\gamma q_S^t, f_S)(\w^t)$ define descent directions for the composition
    $h_S = \ell_S \circ f_S$ at $\w^t$.
\label{prop:descent_direction}
\end{proposition}
In practice, we never compute the exact direction. We show below that the
approximate directions obtained by the conjugate gradient method run in the
primal or in the dual define descent directions for any number of iterations.
Proofs are presented in Appendix~\ref{app:proof_descent_dir}. 

\begin{proposition}\label{prop:descent_dir_cg_primal} Let $\d_S^{t, \tau}$ be
  the direction obtained after $\tau$ iterations of the conjugate gradient
  method, for solving line~\ref{line:descent_dir_primal} of
  Algo.~\ref{algo:proxlin_template} (primal) or line~\ref{line:descent_dir_dual} of
  Algo.~\ref{algo:dual_proxlin}, followed by line~\ref{line:dual_mapping} (dual). 
  Then $-\d_S^{t, \tau}$ is a descent direction for 
  $h_S = \ell_S\circ f_S$ at $\w^t$. 
\end{proposition}

\paragraph{Computational cost.} 

The computational costs associated to the resolution of the inner subproblems in
the primal and dual formulations depend on (i) the cost of each iteration of the
inner solver, (ii) the convergence rate associated to the subproblem. For
quadratic inner subproblems, arising for example with a squared loss, we detail
in Appendix~\ref{app:comput_cost} that the cost of running $\tau$ inner
iterations of CG in the primal or the dual are respectively
\begin{align*}
  \mbox{Primal:} & \ \tau(\mathcal{T}_{\mathrm{JVP-VJP}} + O(p)),  \\
\ \mbox{Dual}: & \ \tau(\mathcal{T}_{\mathrm{JVP-VJP}} + O(m k)),
\end{align*}
where $\mathcal{T}_{\mathrm{JVP-VJP}}$ denotes the computational complexity of a
call to a JVP $\partial \f_S(\w)$ and a call to a VJP $\partial \f_S(\w)^*$.
Therefore, the dual formulation, which operates in the dual space leads to less
expansive inner updates. In addition, the convergence rate of CG on the inner
subproblems depend theoretically on their condition numbers and the distribution
of eigenvalues of the operators considered as detailed in
Appendix~\ref{app:comput_cost}. Since we consider in practice only few inner
iterations, the convergence rate of the method has much less influence than the
per-iteration cost.

\section{Experiments}\label{sec:exp}

We consider image classification tasks with a \textbf{logistic loss} using a
prox-linear direction approximated via the dual formulation of the quadratic
approximation of the loss, using a conjugate gradient method with 2 inner
iterations. We use iterates of the form $\w^{t+1} = \w^t - \d( \gamma^t q_S^t,
f_S)(\w^t)$  denoted \textbf{SPL} for \textbf{stochastic prox-linear} and
iterates of the form $\w^{t+1} = \w^t - \eta^t \d(q_S^t, f_S)(\w^t)$ for
$\eta^t$ chosen by an Armijo line-search, denoted {\bf Armijo SPL}. 

For preliminary diagnosis, we compare the performance of Armijo SPL to several
stochastic gradient based optimization schemes to classify images from the
CIFAR10 datset~\citep{krizhevsky2009learning} with a three layer ConvNet
presented in Appendix~\ref{app:exp} in terms of epochs. Stepsizes are searched
on a log10 scale with a batch-size 256.

\subsection{Prox-linear vs. stochastic gradient}

In Figure~\ref{fig:proxlin-x}, we employ the prox-linear direction as a
replacement for the stochastic gradient in existing algorithms, ranging from
ADAM to SGD with momentum or AdaFactor. Namely, we simply replace
$\nabla h_S(\w^t)$
by
$\d(q_S^t, f_S)(\w^t)$
in each solver's update rule.
Results in time are in Appendix~\ref{app:exp}.

We observe that for most update rules, SPL generally
speeds up the convergence or performs on par in terms of epochs, except for
AdaFactor for which using prox-linear directions perform similarly to gradients. 

\subsection{Dual vs. primal}

In Figure~\ref{fig:compa_time}, we compare the runtime performance of 
the prox-linear direction, depending on whether the primal or the dual was used.
In both cases, we use an Armijo-SPL. While both approaches perform
similarly in terms of iterations, we observe that the dual approach brings some
gains in total runtime.

\begin{figure}[t]
  \centering
  \includegraphics[width=0.95\linewidth]{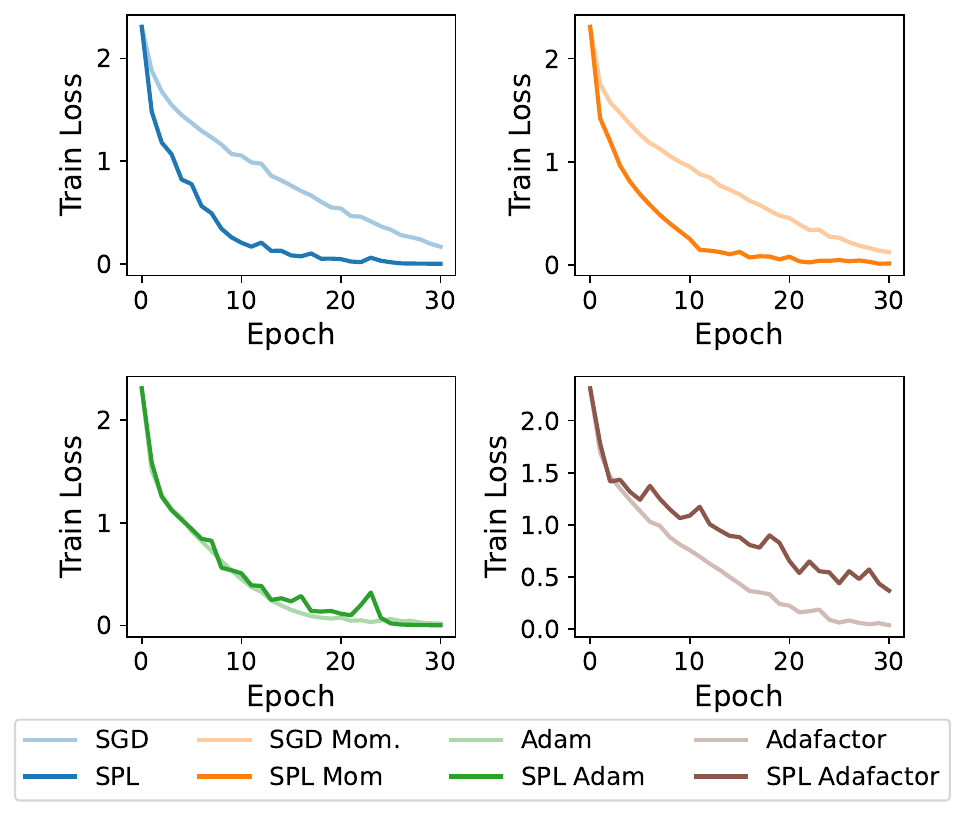}
  \caption{Benefit of using prox-linear directions as replacement for 
  stochastic gradients in existing solvers (CIFAR10, ConvNet). 
  \label{fig:proxlin-x}
  }
\end{figure}
\begin{figure}[t]
    \centering
    \includegraphics[width=0.95\linewidth]{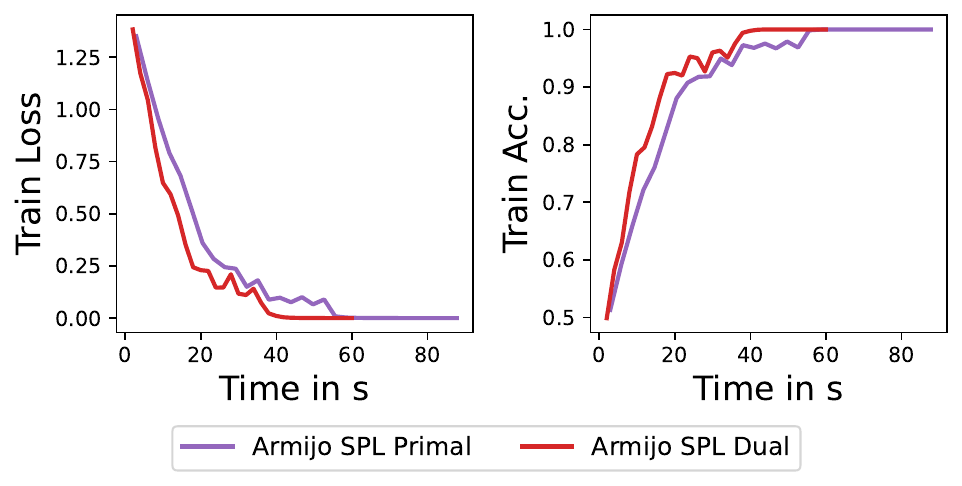}
    \caption{Runtime comparison between primal-based and dual-based
    computations (CIFAR10, ConvNet). \label{fig:compa_time}}
\end{figure}

\begin{figure}[t]
  \centering
  \includegraphics[width=0.95
  \linewidth]{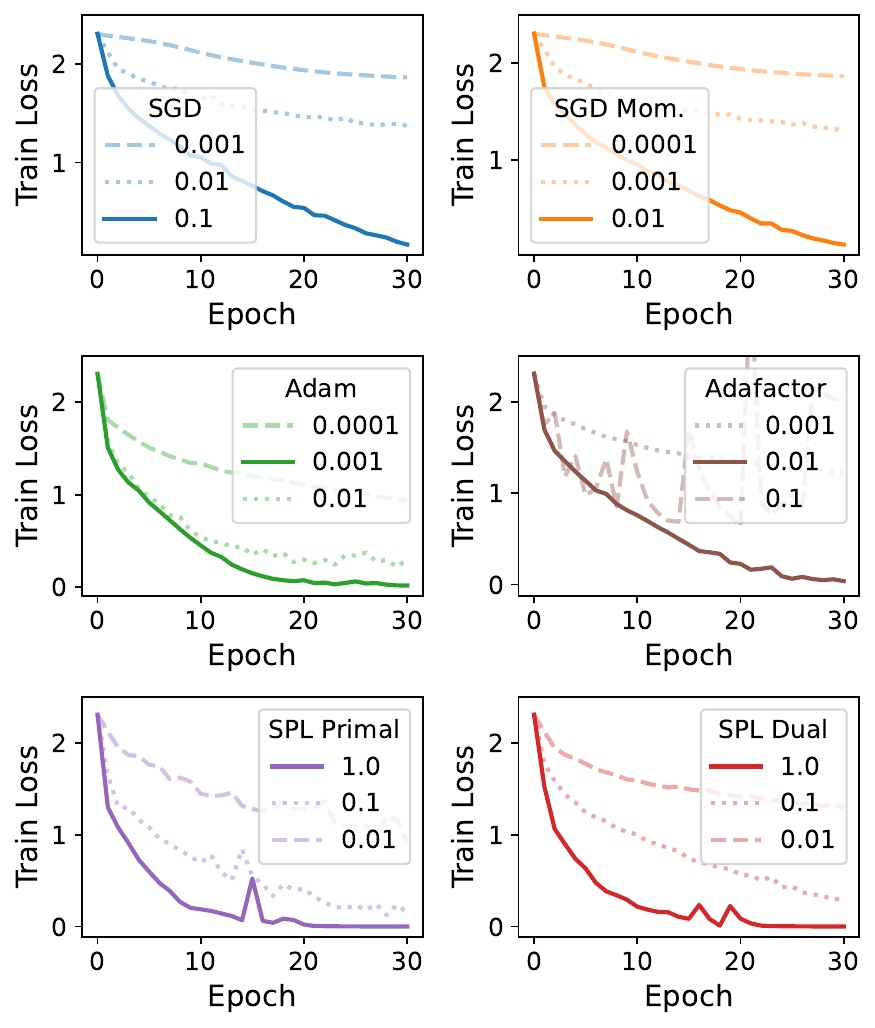}
  \caption{Robustness to stepsize\label{fig:sensitivity} (CIFAR10, ConvNet).
  }
\end{figure}
\subsection{Robustness to hyperparameters}

We analyze the sensitivity to hyperparameters of algorithms based on the
prox-linear direction.
Our goal here is to understand the trade-offs involved in using prox-linear
directions. 

\paragraph{Robustness to stepsize.}

We study in Figure~\ref{fig:sensitivity} the robustness to the stepsize
selection of different algorithms. For each algorithm, we display in solid line
the best stepsize found among $(10^{i})_{i=-4}^0$ and in transparent lines the
two other best stepsizes. Here we use prox-linear udpates with varying
``inner-stepsize'' $\gamma$ to analyze the benefits of
incorporating the stepsize inside the oracle with $\gamma$ chosen in
$(10^{i})_{i=-4}^0$. Namely, here we consider updates $\w^{t+1} = \w^t -
\d(\gamma q_S^t, f_S)(\w^t)$ , denoted {\bf SPL}, rather than $\w^{t+1} = \w^t -
\eta^t \d(q_S^t, f_S)(\w^t)$.

We observe that SPL provides competitive performance for a
larger range of stepsizes than other algorithms such as SGD, SGD with
momentum, or AdaFactor, while exhibting a smilar robustness as Adam.

\paragraph{Robustness to number of inner iterations and batch size.}

One of the hyper-parameters of the algorithm is the number of inner iterations. On
the top row of Figure~\ref{fig:robust_hyperparam}, we analyzed the behavior of
{\bf Armijo SPL} when varying the number of inner iterations. We did not observe
an important sensitivity in this setting. 

On the bottom row of Figure~\ref{fig:robust_hyperparam}, we observe that for too
small or too large mini-batch sizes, {\bf Armijo SPL} does not perform as well
as for medium sizes. Indeed, if the batch is too small, the variance may be too
big. On the other hand, if the batch is too big, since the batch-size influences
the conditioning of subproblem, making only a few steps of the subroutine may
not be sufficient to get a good direction.

\begin{figure}[t]
    \centering
    \includegraphics[width=0.95\linewidth]{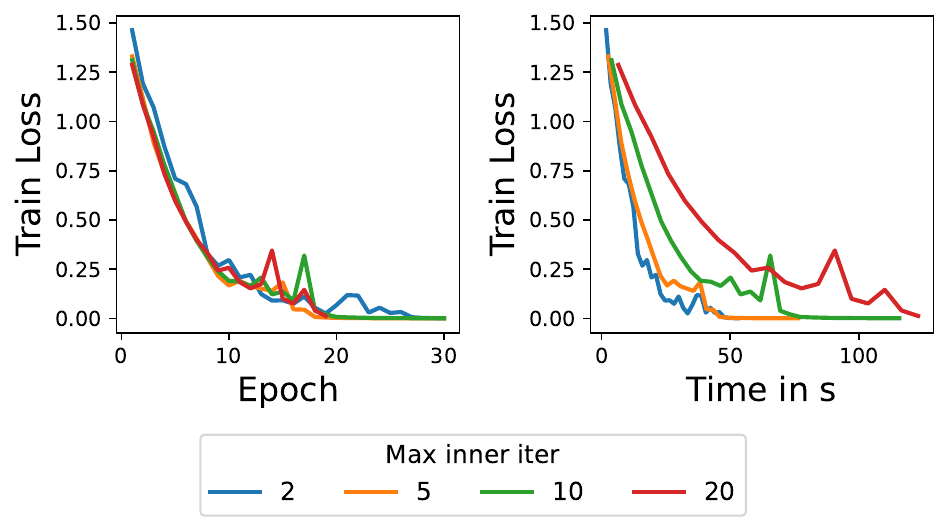}
    \includegraphics[width=0.95\linewidth]{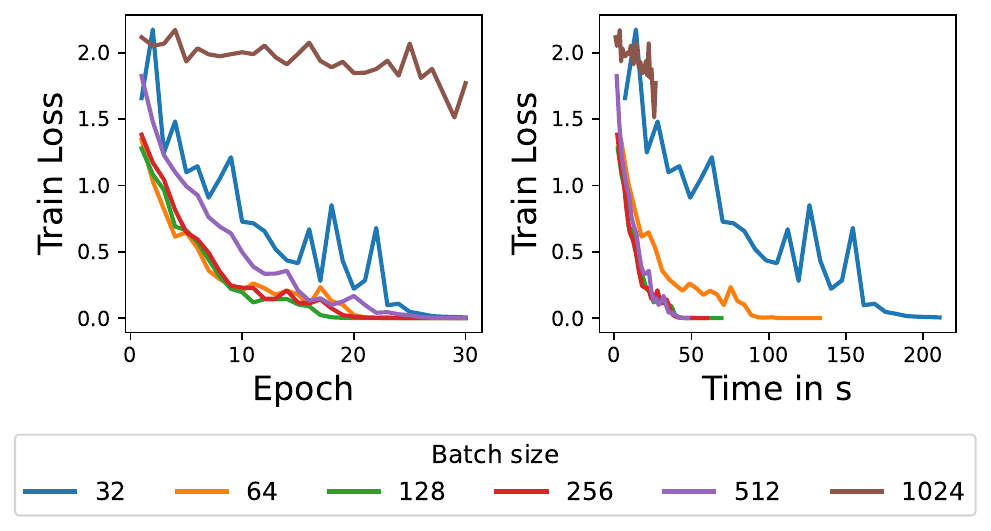}
    \caption{Robustness to \#inner iterations (top) and batch size
        (bottom), on CIFAR10 with a ConvNet. Left: epochs. Right: wallclock time.
        \label{fig:robust_hyperparam} 
    }
\end{figure}

\subsection{Comparison with existing algorithms}

Initial comparisons in Figure~\ref{fig:proxlin-x} demonstrated the benefits of
SPL in terms of epochs. In Figure~\ref{fig:cifar10}, we compare the performance
of Armijo SPL to several stochastic gradient based optimization on the same
ConvNet on CIFAR10 in terms of wallclock time. 
In this experiment, Armijo SPL and SPL are able to reach higher
test accuracy overall. However, even with only two inner CG steps, SPL remains
twice longer to reach convergence.

Additional experiments on ResNets on Imagenet and various ConvNets on CIFAR10
are presented in Appendix~\ref{app:additional_exp}. These additional experiments
confirm with the ones presented here: while our results are promising, adaptive
optimizers like ADAM remain very competitive thanks to their low computational
cost. In Appendix \ref{app:additional_exp}, we show that our approach is
competitive with Shampoo and KFAC. In particular, for small batch sizes, our
approach is faster than KFAC in time. The sensitivity of such methods to the
batch size remains a challenge, to make them more widely applicable.

\begin{figure}[t]
  \centering
  \includegraphics[width=\linewidth]{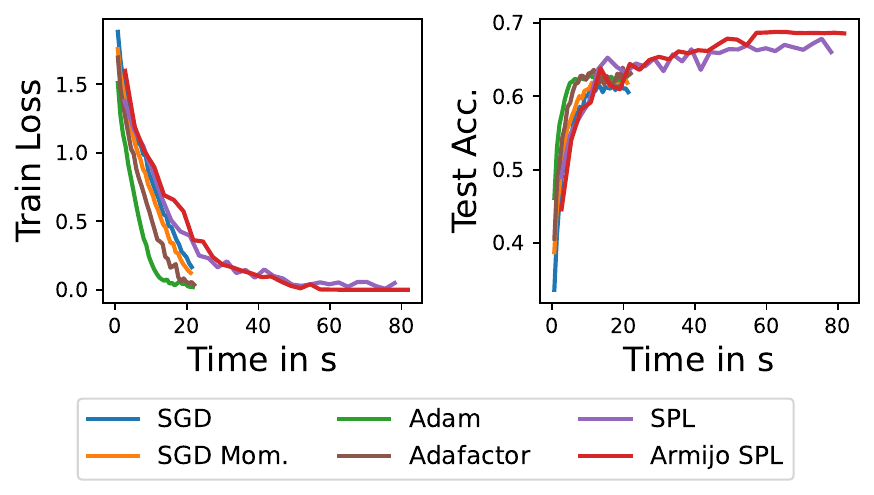}
  \caption{CIFAR-10 with a ConvNet. \label{fig:cifar10}}
\end{figure}

\section{Conclusion}

Prox-linear (a.k.a. modified Gauss-Newton) directions can be thought as a
\textbf{principled middle ground} between the stochastic gradient and the
stochastic (regularized) Newton direction. 
In this paper, we derived the \textbf{Fenchel-Rockafellar dual} associated with
the corresponding subproblem, which to our knowledge had not been studied
before. 
To solve the dual when a quadratic approximation of the loss is used, we
proposed a \textbf{conjugate gradient} algorithm, that integrates seamlessly
with autodiff through the use of \textbf{linear operators} and can handle
\textbf{equality constraints}. We proved that this algorithm produces
\textbf{descent directions}, when run for \textbf{any number} of inner
iterations. Empirically, we found that prox-linear directions work best with
\textbf{medium batch sizes} and are more \textbf{robust} than stochastic
gradients to \textbf{stepsize specification}. 

While we
demonstrated promising results, first-order adaptive methods such as ADAM remain
an excellent trade-off between accuracy and computational cost. Further work is
needed to reduce the computational cost of Gauss-Newton methods in general and
of our dual approach in particuliar. We hope this study brings new insights and
is a first step in that direction.
To conclude, we point out that our approach can
be easily extended to incorpore a \textbf{non-differentiable regularizer},
enabling the use of sparsity-inducing penalties on the network weights, as
studied in Appendix \ref{app:non_diff_reg}.

\newpage
\bibliographystyle{apalike}
\bibliography{gn_refs}

\begin{thebibliography}{}

\bibitem[str, 1991]{strakovs1991real}
 (1991).
\newblock On the real convergence rate of the conjugate gradient method.
\newblock {\em Linear algebra and its applications}, 154:535--549.

\bibitem[Amari, 1998]{amari1998natural}
Amari, S.-I. (1998).
\newblock Natural gradient works efficiently in learning.
\newblock {\em Neural computation}, 10(2):251--276.

\bibitem[Anil et~al., 2020]{anil2020scalable}
Anil, R., Gupta, V., Koren, T., Regan, K., and Singer, Y. (2020).
\newblock Scalable second order optimization for deep learning.
\newblock {\em arXiv preprint arXiv:2002.09018}.

\bibitem[Beck, 2017]{beck2017first}
Beck, A. (2017).
\newblock {\em First-order methods in optimization}.
\newblock SIAM.

\bibitem[Bergou et~al., 2020]{bergou2020convergence}
Bergou, E.~H., Diouane, Y., and Kungurtsev, V. (2020).
\newblock Convergence and complexity analysis of a {Levenberg-Marquardt}
  algorithm for inverse problems.
\newblock {\em Journal of Optimization Theory and Applications}, 185:927--944.

\bibitem[Bj{\"o}rck, 1996]{bjorck1996numerical}
Bj{\"o}rck, A. (1996).
\newblock {\em Numerical methods for least squares problems}.
\newblock SIAM.

\bibitem[Blondel et~al., 2020]{blondel2020learning}
Blondel, M., Martins, A.~F., and Niculae, V. (2020).
\newblock Learning with {Fenchel-Young} losses.
\newblock {\em Journal of Machine Learning Research}, 21(1):1314--1382.

\bibitem[Botev et~al., 2017]{botev2017practical}
Botev, A., Ritter, H., and Barber, D. (2017).
\newblock Practical {Gauss-Newton} optimisation for deep learning.
\newblock In {\em International Conference on Machine Learning}, pages
  557--565.

\bibitem[Bradbury et~al., 2018]{bradbury2018jax}
Bradbury, J., Frostig, R., Hawkins, P., Johnson, M.~J., Leary, C., Maclaurin,
  D., Necula, G., Paszke, A., Vander{P}las, J., Wanderman-{M}ilne, S., and
  Zhang, Q. (2018).
\newblock {JAX}: composable transformations of {P}ython+{N}um{P}y programs.

\bibitem[Burke, 1985]{burke1985descent}
Burke, J.~V. (1985).
\newblock Descent methods for composite nondifferentiable optimization
  problems.
\newblock {\em Mathematical Programming}, 33:260--279.

\bibitem[Diehl and Messerer, 2019]{diehl2019local}
Diehl, M. and Messerer, F. (2019).
\newblock Local convergence of generalized {Gauss-Newton} and sequential convex
  programming.
\newblock In {\em 2019 IEEE 58th Conference on Decision and Control (CDC)},
  pages 3942--3947. IEEE.

\bibitem[Drusvyatskiy and Paquette, 2019]{drusvyatskiy2019efficiency}
Drusvyatskiy, D. and Paquette, C. (2019).
\newblock Efficiency of minimizing compositions of convex functions and smooth
  maps.
\newblock {\em Mathematical Programming}, 178:503--558.

\bibitem[Duchi and Ruan, 2018]{duchi2018stochastic}
Duchi, J.~C. and Ruan, F. (2018).
\newblock Stochastic methods for composite and weakly convex optimization
  problems.
\newblock {\em SIAM Journal on Optimization}, 28(4):3229--3259.

\bibitem[Elfwing et~al., 2018]{elfwing2018sigmoid}
Elfwing, S., Uchibe, E., and Doya, K. (2018).
\newblock Sigmoid-weighted linear units for neural network function
  approximation in reinforcement learning.
\newblock {\em Neural networks}, 107:3--11.

\bibitem[Gargiani et~al., 2020]{gargiani2020promise}
Gargiani, M., Zanelli, A., Diehl, M., and Hutter, F. (2020).
\newblock On the promise of the stochastic generalized {Gauss-Newton} method
  for training {DNNs}.
\newblock {\em arXiv preprint arXiv:2006.02409}.

\bibitem[George et~al., 2018]{george2018fast}
George, T., Laurent, C., Bouthillier, X., Ballas, N., and Vincent, P. (2018).
\newblock Fast approximate natural gradient descent in a {Kronecker} factored
  eigenbasis.
\newblock {\em Advances in Neural Information Processing Systems}, 31.

\bibitem[Gower et~al., 2022]{gower2022cutting}
Gower, R.~M., Blondel, M., Gazagnadou, N., and Pedregosa, F. (2022).
\newblock Cutting some slack for {SGD} with adaptive {Polyak} stepsizes.
\newblock {\em arXiv preprint arXiv:2202.12328}.

\bibitem[Greenbaum, 1997]{greenbaum1997iterative}
Greenbaum, A. (1997).
\newblock {\em Iterative methods for solving linear systems}.
\newblock SIAM.

\bibitem[Gupta et~al., 2018]{gupta2018shampoo}
Gupta, V., Koren, T., and Singer, Y. (2018).
\newblock Shampoo: Preconditioned stochastic tensor optimization.
\newblock In {\em International Conference on Machine Learning}, pages
  1842--1850.

\bibitem[He et~al., 2016]{he2016deep}
He, K., Zhang, X., Ren, S., and Sun, J. (2016).
\newblock Deep residual learning for image recognition.
\newblock In {\em Proceedings of the IEEE conference on Computer Vision and
  Pattern Recognition}, pages 770--778.

\bibitem[Henriques et~al., 2019]{henriques2019small}
Henriques, J.~F., Ehrhardt, S., Albanie, S., and Vedaldi, A. (2019).
\newblock Small steps and giant leaps: Minimal {Newton} solvers for deep
  learning.
\newblock In {\em Proceedings of the IEEE/CVF International Conference on
  Computer Vision}, pages 4763--4772.

\bibitem[Herring et~al., 2019]{herring2019gauss}
Herring, J.~L., Nagy, J., and Ruthotto, L. (2019).
\newblock {Gauss--Newton} optimization for phase recovery from the bispectrum.
\newblock {\em IEEE Transactions on Computational Imaging}, 6:235--247.

\bibitem[Huang and Fu, 2019]{huang2019low}
Huang, K. and Fu, X. (2019).
\newblock Low-complexity proximal {Gauss-Newton} algorithm for nonnegative
  matrix factorization.
\newblock In {\em 2019 IEEE Global Conference on Signal and Information
  Processing (GlobalSIP)}, pages 1--5. IEEE.

\bibitem[Kelley, 1999]{kelley1999iterative}
Kelley, C.~T. (1999).
\newblock {\em Iterative methods for optimization}.
\newblock SIAM.

\bibitem[Kingma and Ba, 2015]{kingma2015adam}
Kingma, D.~P. and Ba, J. (2015).
\newblock Adam: {A} method for stochastic optimization.
\newblock In {\em International Conference on Learning Representations}.

\bibitem[Krizhevsky et~al., 2009]{krizhevsky2009learning}
Krizhevsky, A., Hinton, G., et~al. (2009).
\newblock Learning multiple layers of features from tiny images.

\bibitem[Levenberg, 1944]{levenberg1944method}
Levenberg, K. (1944).
\newblock A method for the solution of certain non-linear problems in least
  squares.
\newblock {\em Quarterly of applied mathematics}, 2(2):164--168.

\bibitem[Lin et~al., 2018]{lin2018catalyst}
Lin, H., Mairal, J., and Harchaoui, Z. (2018).
\newblock Catalyst acceleration for first-order convex optimization: from
  theory to practice.
\newblock {\em Journal of Machine Learning Research}, 18(1):7854--7907.

\bibitem[Liu et~al., 2020]{liu2020ddpnopt}
Liu, G.-H., Chen, T., and Theodorou, E.~A. (2020).
\newblock {DDPNOpt}: Differential dynamic programming neural optimizer.
\newblock {\em arXiv preprint arXiv:2002.08809}.

\bibitem[Marquardt, 1963]{marquardt1963algorithm}
Marquardt, D.~W. (1963).
\newblock An algorithm for least-squares estimation of nonlinear parameters.
\newblock {\em Journal of the society for Industrial and Applied Mathematics},
  11(2):431--441.

\bibitem[Martens, 2010]{martens2010deep}
Martens, J. (2010).
\newblock Deep learning via {Hessian}-free optimization.
\newblock In {\em International Conference on Machine Learning}, volume~27,
  pages 735--742.

\bibitem[Martens, 2020]{martens2020new}
Martens, J. (2020).
\newblock New insights and perspectives on the natural gradient method.
\newblock {\em Journal of Machine Learning Research}, 21(1):5776--5851.

\bibitem[Martens and Grosse, 2015]{martens2015optimizing}
Martens, J. and Grosse, R. (2015).
\newblock Optimizing neural networks with {Kronecker}-factored approximate
  curvature.
\newblock In {\em International conference on machine learning}, pages
  2408--2417.

\bibitem[Martens and Sutskever, 2011]{martens2011learning}
Martens, J. and Sutskever, I. (2011).
\newblock Learning recurrent neural networks with {Hessian}-free optimization.
\newblock In {\em International Conference on Machine Learning}, pages
  1033--1040.

\bibitem[Messerer et~al., 2021]{messerer2021survey}
Messerer, F., Baumg{\"a}rtner, K., and Diehl, M. (2021).
\newblock Survey of sequential convex programming and generalized
  {Gauss-Newton} methods.
\newblock {\em ESAIM: Proceedings and Surveys}, 71:64--88.

\bibitem[Nesterov, 2007]{nesterov2007modified}
Nesterov, Y. (2007).
\newblock Modified {Gauss-Newton} scheme with worst case guarantees for global
  performance.
\newblock {\em Optimisation methods and software}, 22(3):469--483.

\bibitem[Nocedal and Wright, 1999]{nocedal1999numerical}
Nocedal, J. and Wright, S.~J. (1999).
\newblock {\em Numerical optimization}.
\newblock Springer.

\bibitem[Parikh et~al., 2014]{parikh2014proximal}
Parikh, N., Boyd, S., et~al. (2014).
\newblock Proximal algorithms.
\newblock {\em Foundations and trends{\textregistered} in Optimization},
  1(3):127--239.

\bibitem[Paszke et~al., 2017]{paszke2017automatic}
Paszke, A., Gross, S., Chintala, S., Chanan, G., Yang, E., DeVito, Z., Lin, Z.,
  Desmaison, A., Antiga, L., and Lerer, A. (2017).
\newblock Automatic differentiation in pytorch.

\bibitem[Pauloski et~al., 2020]{pauloski2020convolutional}
Pauloski, J.~G., Zhang, Z., Huang, L., Xu, W., and Foster, I.~T. (2020).
\newblock Convolutional neural network training with distributed k-fac.
\newblock In {\em SC20: International Conference for High Performance
  Computing, Networking, Storage and Analysis}, pages 1--12. IEEE.

\bibitem[Pillutla et~al., 2023]{pillutla2023modified}
Pillutla, K., Roulet, V., Kakade, S., and Harchaoui, Z. (2023).
\newblock Modified {Gauss-Newton} algorithms under noise.
\newblock {\em arXiv preprint arXiv:2305.10634}.

\bibitem[Pillutla et~al., 2019]{pillutla2019smoother}
Pillutla, K., Roulet, V., Kakade, S.~M., and Harchaoui, Z. (2019).
\newblock A smoother way to train structured prediction models.
\newblock {\em arXiv preprint arXiv:1902.03228}.

\bibitem[Ren and Goldfarb, 2019]{ren2019efficient}
Ren, Y. and Goldfarb, D. (2019).
\newblock Efficient subsampled {Gauss-Newton} and natural gradient methods for
  training neural networks.
\newblock {\em arXiv preprint arXiv:1906.02353}.

\bibitem[Repetti et~al., 2014]{repetti2014nonconvex}
Repetti, A., Chouzenoux, E., and Pesquet, J.-C. (2014).
\newblock A nonconvex regularized approach for phase retrieval.
\newblock In {\em 2014 IEEE International Conference on Image Processing
  (ICIP)}, pages 1753--1757. IEEE.

\bibitem[Rockafellar, 1976]{rockafellar1976monotone}
Rockafellar, R.~T. (1976).
\newblock Monotone operators and the proximal point algorithm.
\newblock {\em SIAM journal on control and optimization}, 14(5):877--898.

\bibitem[Roulet et~al., 2022]{roulet2022complexity}
Roulet, V., Srinivasa, S., Fazel, M., and Harchaoui, Z. (2022).
\newblock Complexity bounds of iterative linear quadratic optimization
  algorithms for discrete time nonlinear control.
\newblock {\em arXiv preprint arXiv:2204.02322}.

\bibitem[Shalev-Shwartz and Zhang, 2013]{shalev2013stochastic}
Shalev-Shwartz, S. and Zhang, T. (2013).
\newblock Stochastic dual coordinate ascent methods for regularized loss
  minimization.
\newblock {\em Journal of Machine Learning Research}, 14(1).

\bibitem[Sideris and Bobrow, 2005]{sideris2005efficient}
Sideris, A. and Bobrow, J.~E. (2005).
\newblock An efficient sequential linear quadratic algorithm for solving
  nonlinear optimal control problems.
\newblock In {\em Proceedings of the 2005, American Control Conference, 2005.},
  pages 2275--2280. IEEE.

\bibitem[Sutskever et~al., 2013]{sutskever2013importance}
Sutskever, I., Martens, J., Dahl, G., and Hinton, G. (2013).
\newblock On the importance of initialization and momentum in deep learning.
\newblock In {\em International Conference on Machine Learning}, pages
  1139--1147.

\bibitem[Tran-Dinh et~al., 2020]{tran2020stochastic}
Tran-Dinh, Q., Pham, N., and Nguyen, L. (2020).
\newblock Stochastic {Gauss-Newton} algorithms for nonconvex compositional
  optimization.
\newblock In {\em International Conference on Machine Learning}, pages
  9572--9582.

\bibitem[Vaswani et~al., 2019]{vaswani2019painless}
Vaswani, S., Mishkin, A., Laradji, I., Schmidt, M., Gidel, G., and
  Lacoste-Julien, S. (2019).
\newblock Painless stochastic gradient: Interpolation, line-search, and
  convergence rates.
\newblock {\em Advances in neural information processing systems}, 32.

\bibitem[Yu and Wilamowski, 2018]{yu2018levenberg}
Yu, H. and Wilamowski, B.~M. (2018).
\newblock {Levenberg--Marquardt} training.
\newblock In {\em Intelligent systems}, pages 12--1. CRC Press.

\bibitem[Zhang and Xiao, 2021]{zhang2021stochastic}
Zhang, J. and Xiao, L. (2021).
\newblock Stochastic variance-reduced prox-linear algorithms for nonconvex
  composite optimization.
\newblock {\em Mathematical Programming}, pages 1--43.

\bibitem[Zhang et~al., 2022]{zhang2022scalable}
Zhang, L., Shi, S., Wang, W., and Li, B. (2022).
\newblock Scalable k-fac training for deep neural networks with distributed
  preconditioning.
\newblock {\em IEEE Transactions on Cloud Computing}.

\end{thebibliography}

\newpage
\appendix
\onecolumn

\begin{center}
  {\Large \bf Dual Gauss-Newton Directions for Deep Learning} \\[1em]
  {\large Appendix}
\end{center}

Appendix~\ref{app:related_work} expands on related work.
Appendix~\ref{app:non_diff_reg} presents how to incorporate additional
regularizers into the computation of the Gauss-Newton direction.
Appendix~\ref{app:proofs} contains all proofs. Appendix~\ref{app:comput_cost}
details the computational costs. Finally Appendix~\ref{app:exp} presents some
additional experiments and experimental details.

\section{Detailed literature review}\label{app:related_work}

\paragraph{Variational perspective on optimization oracles.}

The variational perspective on optimization oracles is well-known, see, e.g.,
\citet{parikh2014proximal,beck2017first}. This perspective can be used together
with full linearization to motivate SGD variants based on Polyak stepsize
\citep{gower2022cutting}. 

In addition to the interpretation of prox-linear/Gauss-Newton as an intermediate
oracle between gradient and Newton updates, the prox-linear can be interpreted
through the lense of proximal point methods. With the notations of
Section~\ref{sec:intro}, the proximal point algorithm
\citep{rockafellar1976monotone} computes the next iterate as
\begin{align*}
\w^{t+1} 
&\coloneqq \prox(\gamma \obj_i)(\w^t) \\
&\coloneqq \argmin_{\w \in \RR^p} \obj_i(\w) + \frac{1}{2\gamma}\|\w -
\w^t\|_2^2 \\
&= \w^t - \argmin_{\d \in \RR^p} \obj_i(\w^t - \d) +
\frac{1}{2\gamma}\|\d\|_2^2,
\end{align*}
where we used the change of variable $\w \coloneqq \w^t - \d$.
Unfortunately, when $h_i = \ell_i \circ f_i$ and $f_i$ is a neural network,
this subproblem is nonconvex.
In comparison, 
we can write the prox-linear update as
\begin{align*}
\mathrm{prox\_linear}(\gamma \ell_i, f_i)(\w^t) 
\coloneqq& \prox(\mathrm{plin}(\gamma \ell_i, f_i, \w^t))(\w^t) \\
=& \argmin_{\w \in \RR^p} 
\mathrm{plin}(\ell_i, f_i, \w^t)(\w)
+ \frac{1}{2\gamma}\|\w^t - \w\|_2^2 \\
=& \w^t - \d(\gamma \ell_i, f_i)(\w^t) \approx \prox(\gamma h_i)(\w^t).
\end{align*}
Therefore, we can see the resulting update as the proximal point iteration on
a partially linearized function, hence the name prox-linear.
Importantly however, unlike the proximal point update,
the associated subproblem is convex. 

\paragraph{Nonlinear least-squares (deterministic case).}

The idea of exploiting the compositional structure of an objective to partially
linearize the objective and minimize the resulting simplified subproblem, stems
from the resolution of nonlinear least-squares
problems~\citep{bjorck1996numerical} of the form
\[
  \min_{\w\in \RR^p} \|f(\w)\|_2^2,
\]
where $f$ is a nonlinear differentiable function. In this context, the
Gauss-Newton algorithm~\citep{kelley1999iterative} proceeds originally by
computing 
\[
  \d^t \coloneqq \argmin_{\v \in \RR^p} \|f(\w^t) - \partial f(\w^t)\d \|_2^2,
\]
and computing $\w^{t+1}$ by means of a line-search along the direction
$-\d^t$ (which is guaranteed to be a descent direction). The Gauss-Newton
algorithm can converge locally to a solution at a quadratic rate provided that
the initial point is close enough to the solution~\citep{kelley1999iterative}. 
The Gauss-Newton algorithm can be subject to numerical instability as soon as 
the operator $\partial f(\w^t)$ is non-singular. To circumvent this issue, 
\citet{levenberg1944method} and
\citet{marquardt1963algorithm} introduced the now called Levenberg-Marquardt algorithm 
that computes directions according to a regularized version of the Gauss-Newton 
direction
\[
    \d^t \coloneqq \argmin_{\d \in \RR^p} \|f(\w^t) 
    - \partial f(\w^t)\d \|_2^2 + \frac{\lambda}{2} \|\d\|_2^2.
\]
A line-search along $-\d^t$ is then taken to define the next iterate (the
direction $-\d^t$ is again a descent direction). The parameter $\lambda$
(equivalent to $1/\gamma$ in the derivation of our algorithm) acts as a
regularization to ensure that the direction is relevant at the current iterate:
larger $\lambda$ induce smaller directions, closer to a negative gradient
direction. The parameter $\lambda$ may be modified along the iterations by a
trust-region mechanism~(see e.g. \citet{bergou2020convergence}). Namely a
trust-region mechanism computes
\[
  r_t \coloneqq \frac{||f(\w^t)||^2_2 - \|f(\w^t - \d^t)\|^2_2}{m_t({\bm 0}) - m_t(\d^t)},
\]
where $m_t(\d) \coloneqq \|f(\w^t) - \partial f(\w^t)\d \|_2^2 + \lambda_t
\|\d\|_2^2/2$ the approximate model of the objective. Here, $r_t$ is a measure
of how good the model was for the current $\lambda_t$. For $r_t \gg 0$, the step
taken reduced efficiently the original objective, while for, e.g., $r_t\ll 0$,
the model provided a bad direction. At each iteration, if $r_t> \delta_1$,
$\lambda_{t+1} = 0.5\lambda_t$ for example, is decreased, if $\rho_t <
\delta_2$, the iteration is a priori redone with an increased $\lambda_t$, or we
may simply set $\lambda_{t+1} = 2 \lambda_t$ for the next iteration.
Gauss-Newton-like algorithms have been applied succesfully to
phase-retrieval~\citep{herring2019gauss,repetti2014nonconvex}, nonlinear
control~\citep{sideris2005efficient}, non-negative matrix
factorization~\citep{huang2019low} to cite a few. 

\paragraph{Compositional problems (deterministic case).}

Gauss-Newton-like algorithms have then been generalized beyond nonlinear
least-squares to tackle generic compositional problems of the form $\ell\circ f
+ \rho$ with $\ell$ convex, $f$ differentiable nonlinear, $\rho$ a simple
function with computable proximity operator. The resulting algorithm, called
prox-linear~\citep{burke1985descent} computes the next oracle as 
\[
  \w^{t+1} \coloneqq  \argmin_{\w\in \RR^p} \ell(f(\w^t) 
  + \partial f(\w^t)(\w-\w^t)) + \rho(\w) + \frac{\lambda}{2} \|\w- \w^t\|_2^2.
\]
\citet{nesterov2007modified} proposed to minimize nonlinear residuals with a
generic sharp metric such as $\ell = \|\cdot\|_2$ with $\rho = 0$.
\citet{nesterov2007modified} proved global convergence rates of the above method
given that $\sigma_{\min}(\partial f(\w)^*) \geq \sigma >0$ for any $\w$ and
also gave local convergence rates for $\sigma_{\min}(\partial f(\w)) \geq \sigma
>0$ around an initial point close to a local solution.
\citet{drusvyatskiy2019efficiency} considered prox-linear algorithms for finite
sum objectives, that is, problems of the form $\frac{1}{n}\sum_{i=1}^n \ell_i
\circ f_i + \rho$. \citet{drusvyatskiy2019efficiency} considered the norm of the
scaled difference between iterates as a stationary measure of the algorithm by
relating it to the gradient of the Moreau envelope of the objective. They
proposed to solve the sub-problem up to a near-stationarity criterion defined by
the norm of the (sub)-gradient of the dual objective of the sub-problem. They
derived the total computational complexity of the algorithm when using various
inner solvers from an accelerated gradient algorithm to fast incremental solvers
such as SVRG or its accelerated version. Rates are provided for the case where
$\ell$ is smoooth or non-smooth but smoothable. 
\citet{pillutla2019smoother} evaluated these algorithms in the context of
structured prediction with smoothed oracles. The reported performance was on par
with SGD. \citet{pillutla2023modified} performed more synthetic experiments with
similar conclusions, casting doubt on the usefulness of the method. Note however
that~\citet{pillutla2019smoother, pillutla2023modified} did not consider varying
the regularization, nor a mini-batch version of the algorithm, nor incorporating
the algorithm into other first-order mechanisms.

For compositional problems with general loss $\ell$, an alternative to the
prox-linear algorithm is to use a quadratic approximation of $\ell$ together
with the linearization of $f$ \citep{messerer2021survey, roulet2022complexity}.
\citet{roulet2022complexity} showed the global convergence and local convergence
of such a method with the same assumption as \citet{nesterov2007modified}, i.e.,
$\sigma_{\min}(\partial f(\w)^*) \geq \sigma >0$, while asserting this condition
in some nonlinear control problems. \citet{diehl2019local} considered the local
convergence of generalized Gauss-Newton algorithms under suitable assumptions.

\paragraph{Stochastic case.}

\citet{duchi2018stochastic} studied the asymptotic convergence of stochastic
versions of the proxlinear algorithms. They considered objectives of the form
$\E_{z\sim p}[\ell(f(\cdot; z); z)] + \rho$ for some unknown distribution $p$.
The setting considered in our paper is an instance of such a problem. The
prox-linear algorithm consists then in iterates of the form
\[
    \w^{t+1} =  \argmin_{\w\in \RR^p} 
    \frac{1}{|S|}\sum_{i \in S}\ell(f(\w^t; z_i) 
  + \partial f(\w^t; z_i)(\w-\w^t); z_i) 
  + \rho(\w) + \frac{\lambda}{2} \|\w- \w^t\|_2^2,
\]
for $(z_i)_{i \in S} \stackrel{i.i.d.}{\sim} p$ a mini-batch of samples.
\citet{duchi2018stochastic} presented asymptotic rates of convergence for this
method with experiments on phase retrieval problems. In a slightly different
spirit, namely, objectives of the form $\ell(\E_{z\sim p}[f(\cdot; z)]) + \rho$,
\citet{tran2020stochastic,zhang2021stochastic} presented convergence rates using
estimators of the Jacobian and the function values of $f$, with stochastic
estimators such as SPIDER, SARAH or simply large mini-batches.
\citet{gargiani2020promise} considered stochastic versions of the generalized
Gauss-Newton algorithm for deep learning applications with simple experiments on
MNIST, FashionMNIST an CIFAR10. However, they did not use the dual as we do and
they did not consider using the oracle as a replacement for the gradient in
existing solvers such as ADAM.

\paragraph{Jacobian/Hessian-free Gauss-Newton methods in deep learning.} 

To harness the potential power of second-order optimization algorithms such as a
Newton method, \citet{martens2010deep} considered implementing a Newton method
by accessing Hessian-vector products and inverting the Hessian by a conjugate
gradient method. \citet{martens2010deep} introduced usual techniques from the
Newton method such as damping and trust-region techniques to tune the
regularization. An issue quickly pointed out by \citet{martens2010deep} is that
the Hessian of the network need not be positive definite which means that
$H^{-1}g$ for $H$ the Hessian of the objective and $g$ the gradient does not
necessarily lead to a descent direction. While the Hessian could be regularized
to prevent this issue, this may add a non-negligible overhead. In contrast,
oracles based on partial linearizations \`a la Gauss-Newton provably return a
descent direction. Therefore \citet{martens2010deep} considered in practice a
Gauss-Newton algorithm. \citet{martens2010deep} also considered some
preconditioning techniques for the conjugate gradient method and argues for
large mini-batch sizes, see Section 4 of the aforementioned paper.
\citet{martens2011learning} applied this technique to recurrent neural networks. 

Recently, \citet{ren2019efficient} revised such a technique by solving the
subproblem associated to a generalized Gauss-Newton iteration by means of the
Woodbury formula and using a form of trust-region technique. The Woodbury
formula aims to reduce computational costs associated to inversions by
exploiting the structure of the subproblems. In the case of the squared loss,
the two approaches coincide. However, for the quadratic approximation of the
logistic loss, \citet{ren2019efficient} considered a non-symmetric development
of the Woodbury formula, that differs from our approach. By carefully tackling
the dual formulation of the problem in the case of a quadratic approximation of
the logistic loss, we are able to keep the same inner solver, i.e., a conjugate
gradient method, and to prove that the resulting direction is always a descent
direction. \citet{ren2019efficient} performed experiments on CIFAR10, MNIST and
webspam with two hidden layers MLPS. In a similar spirit,
\citet{henriques2019small} considered a Hessian-free optimization algorithm to
compute exactly a Newton step. \citet{henriques2019small} considered performing
one gradient step on the subproblem by means of a Hessian-vector product and use
the resulting direction in place of the gradient. The algorithm of
\citet{henriques2019small} can therefore be seen as an extreme case of our
algorithm with only one iteration in the subproblem (with the additional
difference that they consider a Newton step). \citet{henriques2019small}
conducted an extensive set of experiments on CIFAR10 with ConvNets, ResNets,
ImageNet with VGG ConvNet and MNIST with MLPS as well as standard difficult
nonconvex objectives such as the Rosenberg function. They observed some gains in
accuracy and speed in epochs. On the other hand, gains in time were only
reported for a small architecture \citep[Figure 3]{henriques2019small}.

\paragraph{Approximating the Gauss-Newton matrix by block diagonal blocks.}

Rather than solving approximately for a Gauss-Newton-like direction, a line of
work starting from~\citet{martens2015optimizing} considered using a block
diagonal approximation of the Gauss-Newton matrix. Such a direction often took
the terminology of natural gradient descent algorithms since, for losses
stemming from an exponential family, a generalized Gauss-Newton method coincides
with a natural gradient descent~\citep{martens2020new}. The resulting algorithm
called KFAC (Kronecker Factored Approximate Curvature) stems from the
observation that the block of the Gauss-Newton matrix corresponding to the
$k$\textsuperscript{th} can be factorized as a Kronecker product whose
expectation may be approximated as the Krinecker product of the expectations.
\citet{george2018fast} extended the KFAC method in an algorithm EKFAC that
further tries to compute an adequate eigenbasis along which to compute the
approximate blocks.
\citet{botev2017practical} considered a finer decomposition of the block
diagonal Gauss-Newton matrix computed by back-propagating the information
through the graph of a feed-forward network leading to the algorithm KFRA. Both
KFAC and KFRA a priori depend on the proposed architecture. They were developed
for MLPs and subsequently extended to convolutional neural
networks~\citep{pauloski2020convolutional} and transformers
architecture~\citep{zhang2022scalable}. Such layer-wise decomposition of
second-order methods is reminiscent of techniques used in nonlinear control to
implement Gauss-Newton methods or their nonlinear control variant called
differentiable dynamic programming which were adapted to a deep learning context
by~\citet{liu2020ddpnopt}. Recently, \citet{gupta2018shampoo,anil2020scalable}
generalized the idea of computing generic preconditioners for deep networks by
exploiting their tensor structure.

\section{Non-differentiable regularizer extension}
\label{app:non_diff_reg}

In this section, we consider regularized objectives of the form
\[
    \min_{\w \in \RR^p}
    \left[\frac{1}{n} \sum_{i=1}^n \obj_i(\w) + \rho(\w)
    = \frac{1}{n} \sum_{i=1}^n \ell_i(f_i(\w)) + \rho(\w)\right],
\]
where $\rho$ is a potentially non-differentiable regularization,
such as the sparsity-inducing penalty $\rho(\w) = \lambda\|\w\|_1$, 
where $\lambda > 0$ controls the regularization strength.

\subsection{Primal}

In order to perform an update on a mini-batch $S$, we can solve
\begin{align*}
\w(\gamma \ell_S, f_S, \gamma \rho)(\w^t)
\coloneqq
&\argmin_{\w \in \RR^p} 
\frac{1}{m} \sum_{i \in S} 
\mathrm{plin}(\ell_i, \f_i, \w^t)(\w)
+ \frac{1}{2\gamma}\|\w^t - \w\|_2^2 + \rho(\w) \\
=&\argmin_{\w \in \RR^p} 
\sum_{i \in S} 
\ell_i(\f_i^t + J_i^t(\w - \w^t))
+ m \left[ \frac{1}{2\gamma}\|\w^t - \w\|_2^2 + \rho(\w)\right].
\end{align*}
Using the change of variable $\w \coloneqq \w^t - \d$, this can equivalently be
written
\begin{equation*}
\w(\gamma \ell_S, f_S, \gamma \rho)(\w^t) = \w^t - \d(\gamma \ell_S,
f_S, \gamma \rho)(\w^t),
\end{equation*}
where
\begin{align}
\d(\gamma \ell_S, f_S, \gamma \rho)(\w^t)
&\coloneqq
\argmin_{\d \in \RR^p} 
\sum_{i \in S} \ell_i(\f_i^t -J_i^t \d) 
+ m\left[ \frac{1}{2\gamma}\|\d\|_2^2 + \rho(\w^t -\d)\right]
\label{eq:prox_primal_subproblem}
\end{align}
For sparsity-inducing penalties such as the $\ell_1$ norm, this is a non-smooth
convex problem that can be solved by a {\bf proximal gradient method}, by using
the proximal operator associated to the regularization $\rho$. We rather
consider an approximation of the above problem allowing for an algorithm that
does not require any additional hyperparameters.

\subsection{Dual}

Let us denote 
$r_t(\d) \coloneqq m \left[\frac{1}{2\gamma} \|\d\|^2_2 + 
\rho(\w^t - \d)\right]$.
From Proposition \ref{prop:dual}, the dual of 
the primal subproblem \eqref{eq:prox_primal_subproblem} is
\begin{equation}\nonumber
\alphav(\gamma \ell_S, f_S, \gamma \rho)(\w^t)
\coloneqq
\argmin_{\alphav_S \in \RR^{m \times k}} 
\ell_S^*(\alphav_S) 
- \langle \f_S^t, \alphav_S \rangle
+ r^*_t\left((J_S^t)^*\alphav_S\right).
\end{equation}
Unfortunately, this subproblem may be difficult to solve in general.
As shown in Proposition \ref{prop:approx_dual_subproblem},
assuming $r_t$ is $\mu$-strongly convex, we can approximate the dual subproblem
around any $\u^t$ with
\[
\alphav(\gamma \ell_S, f_S, \gamma \rho)(\w^t)
\approx
\argmin_{\alphav_S \in \RR^{m \times k}} 
\ell_S^*(\alphav_S) 
- \langle \f_S^t - \deltav_S^t, \alphav_S \rangle
+ \frac{1}{2\mu} \|(J_S^t)^*\alphav_S\|^2,
\]
where
$
  \deltav_S^t \coloneqq J_S^t \left(\nabla r_t^*(\u^t) - \frac{1}{\mu} \u^t\right).
$
We choose $\u^t ={\bm 0}$, so that approximating the computation of $\d(\gamma
\ell_S, f_S, \gamma \rho)(\w^t)$ around $\d = {\bm 0}$ amounts to approximating
the computation of $\w(\gamma \ell_S, f_S, \gamma \rho)(\w^t)$ around $\w^t$.
Note that if $\rho(\w) = 0$, we get $\w^t = \frac{1}{\mu} \u^t$, so that
$\deltav_S^t = \zeros$. Therefore, we recover the dual subproblem
\eqref{eq:minibatch_dual_subproblem} in this case.

The entire procedure is summarized in Algorithm ~\ref{algo:proxlin_dual_prox}.

\subsection{Examples of regularizers}

\paragraph{Quadratic regularization.}

If $r_t(\d) = \frac{m}{2\gamma} \|\d\|^2_2$,
which is strongly convex with constant $\mu = \frac{m}{\gamma}$,
we obtain 
$r_t^*(\u) = \frac{\gamma}{2m} \|\u\|^2_2$ 
and therefore
$\nabla r_t^*(\u) = \frac{\gamma}{m} \u$.
We therefore recover Algorithm \ref{algo:dual}.

\begin{algorithm}[t]
  \caption{Dual-based prox-linear algorithm with $\mu$-strongly convex
  regularization $r_t$ \label{algo:proxlin_dual_prox}}
  \begin{algorithmic}[1]
      \STATE{\bf Inputs:} Parameters $\w^t$, ``inner stepsize'' $\gamma>0$

      \STATE Compute network outputs
      $\f_S^t$, instantiate JVP $J_S^t$ and VJP $(J_S^t)^*$ as in Algorithm~\ref{algo:proxlin_template}
      
      \STATE
      Compute for any $\u^t$, e.g., $\u^t \coloneqq {\bm 0}$,
      \[
        \deltav_S^t \coloneqq J_S^t \left(\nabla r_t^*(\u^t) - \frac{1}{\mu}
      \u^t\right)
      \] 
      \STATE Run inner solver to approximately solve 
  \begin{equation*}
  \alphav_S^t \approx \argmin_{\alphav_S \in \RR^{m \times k}} 
  \ell_S^*(\alphav_S) - \langle \alphav_S, \f_S^t - \deltav_S^t \rangle 
  + \frac{1}{2\mu} \|(J_S^t)^*\alphav_S\|_2^2
  \end{equation*}
  or to approximately solve the equality constrained QP
  \begin{align*}
      \alphav_S^t \approx
    \argmin_{\substack{\alphav \in \RR^{m\times k}}} & 
    \frac{1}{2} \langle (\alphav_S - \g_S^t) ,
    (H_S^t)^{\dagger} (\alphav_S - \g_S^t) \rangle + \langle \alphav_S,
    \deltav_S^t \rangle
    + \frac{1}{2\mu} \|(J_S^t)^* \alphav_S\|_2^2 \\
    \mbox{s.t.} \ & (\idm - H_S^t(H_S^t)^{\dagger})(\alphav_S - \g_S^t) = {\bm 0} ,
  \end{align*}
  with $\g_S^t = \nabla \ell_S(\f_S^t)$, $H_S^t \coloneqq \nabla^2
  \ell_S(\f_S^t)$ and $(H_S^t)^\dagger$ the
  pseudo-inverse~\eqref{eq:pseudo_inverse} (closed forms available for the
  logistic loss).
  
  \STATE Compute direction 
  \[
    \d^t_S \coloneqq \nabla r_t^*((J_S^t)^*
  \alphav_S^t)
  \]

  \STATE Set next parameters $\w^{t+1}$ by 
  \begin{align*}
    \w^{t+1} & \coloneqq \w^t - \d_S^t 
    \quad
    \mathrm{(fixed~stepsize~\eqref{eq:minibatch_update})} 
    \qquad \mbox{or} \qquad 
    \w^{t+1} \coloneqq \w^t - \eta^t \d_S^t 
    \quad \mathrm{(linesearch~\eqref{eq:minibatch_linesearch_update})}
  \end{align*}
  for $\eta^t$ s.t.  $h_S(\w^t - \eta^t \d_S^t)
  \le
  h_S(\w^t) - \beta \eta^t \langle \d_S^t, \g_S^t \rangle.$

  \STATE {\bf Outputs:} $\w^{t+1}$
  \end{algorithmic}
\end{algorithm}

\paragraph{Sum of quadratic and another regularization.}

If $r_t(\d) = \frac{m}{2\gamma} \|\d\|^2_2 + m ~ \rho(\w^t - \d)$, 
which is strongly convex with constant $\mu = \frac{m}{\gamma}$ in general
and with $\mu = \frac{m}{\gamma} + m \lambda$ if $\rho$ is $\lambda$-strongly
convex (which is not required),
using the change of variable $\w \coloneqq \w^t - \d$, we obtain
\begin{align*}
\nabla r_t^*(\u)
&= \argmax_{\d \in \RR^p} \langle \u, \d \rangle - r_t(\d) \\
&= \argmax_{\d \in \RR^p} \langle \u, \d \rangle  
- \frac{m}{2\gamma} \|\d\|^2_2 - m ~ \rho(\w^t - \d) \\
&= \w^t - \argmax_{\w \in \RR^p} -\langle \u, \w \rangle  
- \frac{m}{2\gamma} \|\w^t - \w\|^2_2 - m ~ \rho(\w) \\
&= \w^t - \argmin_{\w \in \RR^p}
\|\w - (\w^t - \frac{\gamma}{m} \u)\|_2^2 + \gamma \rho(\w) \\
&= \w^t - \mathrm{prox}_{\gamma \rho}\left(\w^t - \frac{\gamma}{m} \u\right).
\end{align*}

\paragraph{Sum of quadratic and $L_1$ regularizations.}

As a particular example of the above,
if $\rho(\w) = \lambda \|\w\|_1$, we obtain
\begin{equation*}
\mathrm{prox}_{\gamma \rho}(\z)
= \mathrm{ST}_{\lambda \gamma}(\z),
\end{equation*}
where we defined the soft-thresholding operator
\begin{equation*}
\mathrm{ST}_\tau(\z) \coloneqq
\prox_{\tau \|\cdot\|_1}(\z) = 
\begin{cases}
z_j - \tau, &z_j > \tau\\
0, & |z_j| \leq \tau\\
z_j + \tau, &z_j < -\tau
\end{cases}.
\end{equation*}

\section{Proofs}\label{app:proofs}

\subsection{Derivation of the dual subproblem}
\label{app:dual_subproblem}

We begin by deriving the dual when using a generic strongly convex regularizer
$r_t$.
\begin{proposition}
\label{prop:dual}
Denote 
$\f_i^t \coloneqq f_i(\w^t)$,
$J_i^t \coloneqq \partial f_i(\w^t)$
and $r_t \colon \RR^p \to \RR$ a strongly convex regularizer.
Then,
\begin{equation*}
\min_{\d \in \RR^p} 
\sum_{i \in S} \ell_i(\f_i^t - J_i^t \d) + r_t(\d) 
= -
\min_{\alphav_S \in \RR^{m \times k}}
\sum_{i \in S} \left(\ell_i^*(\alphav_i) 
- \langle \f_i^t, \alphav_i \rangle \right)+ r^*_t\left((J_S^t)^* \alphav_S \right)
\end{equation*}
where 
$(J_S^t)^* \alphav_S \coloneqq \sum_{i=1}^m (J_i^t)^* \alphav_i \in \RR^p$, $\alphav_S = (\alphav_1, \ldots, \alphav_m)^\top \in \RR^{m \times k}$. 

The dual-primal link is 
\begin{equation*}
\d^\star = \nabla r^*_t\left((J_S^t)^* \alphav_S^\star \right), 
\end{equation*}
where
\begin{equation*}
\nabla r^*_t(\u)
= \argmax_{\d \in \RR^p} \langle \u, \d \rangle - r_t(\d).
\end{equation*}
\end{proposition}

\begin{proof}
\begin{equation}\nonumber
\begin{aligned}
&\min_{\d \in \RR^p}\sum_{i \in S} \ell_i(\f_i^t -J_i^t \d ) + r_t(\d) \\
= &\min_{\d \in \RR^p} \sum_{i=1}^m \max_{\alphav_i \in \RR^k}  
\langle \f_i^t - J_i^t \d, \alphav_i \rangle 
- \ell_i^*(\alphav_i) 
+ r_t(\d) \\
= &\max_{\alphav_1, \dots, \alphav_m \in \RR^k} 
\sum_{i \in S}  \left(-\ell_i^*(\alphav_i) 
+ \langle \f_i^t, \alphav_i \rangle \right)
+ \left[\min_{\d \in \RR^p} 
\left\langle - \sum_{i \in S}(J_i^t)^* \alphav_i, \d \right\rangle 
+ r_t(\d)\right] \\
= &\max_{\alphav_1, \dots, \alphav_m \in \RR^k} 
\sum_{i \in S} \left(-\ell_i^*(\alphav_i) 
+ \langle \f_i^t, \alphav_i \rangle \right)
- \left[\max_{\d \in \RR^p} 
\left\langle \sum_{i \in S}(J_i^t)^* \alphav_i, \d \right\rangle 
- r_t(\d)\right] \\
= &\max_{\alphav_1, \dots, \alphav_m \in \RR^k} 
\sum_{i \in S} \left(-\ell_i^*(\alphav_i) 
+ \langle \f_i^t, \alphav_i \rangle  \right)
- r^*_t\left(\sum_{i \in S}(J_i^t)^* \alphav_i\right).
\end{aligned}
\end{equation}
\end{proof}

\subsection{Approximate dual subproblem}
\label{app:approximate_dual_subproblem}

Consider the dual subproblem derived in Proposition~\ref{prop:dual}, that is,
\begin{equation}\nonumber
\min_{\alphav_S \in \RR^{m \times k}} 
\ell_S^*(\alphav_S) 
- \langle \f_S^t, \alphav_S \rangle
+ r^*_t\left((J_S^t)^*\alphav_S\right).
\end{equation}
Unfortunately, this subproblem could be difficult to solve for generic $r_t$. In
our case $r_t$ takes the form $r_t(\d) = m \left[\rho(\w^t -\d) +
\frac{1}{2\gamma}\|\d\|_2^2 \right]$, which is strongly convex. We can therefore
exploit the smoothness of its convex conjugate $r_t^*$. Inspired by the
prox-SDCA algorithm ~\citep{shalev2013stochastic}, we therefore propose to
approximate $r_t^*$ by a quadratic upper bound.
\begin{proposition}\label{prop:approx_dual_subproblem}
  If $r_t$ is $\mu$-strongly convex, the solution of the dual subproblem 
  \begin{equation}\label{eq:dual_prox}
    \min_{\alphav_S \in \RR^{m \times k}} 
        \ell_S^*(\alphav_S)
    - \langle \f_S^t, \alphav_S \rangle
    + r^*_t\left((J_S^t)^*\alphav_S\right),
  \end{equation}
  can be approximated
  around any $\u^t \in \RR^p$
  by solving
  \[
\min_{\alphav_S \in \RR^{m \times k}} 
\ell_S^*(\alphav_S) 
- \langle \f_S^t - J_S^t \left(\nabla r_t^*(\u^t) - \frac{1}{\mu} \u^t\right), \alphav_S \rangle
+ \frac{1}{2\mu} \|(J_S^t)^*\alphav_S\|^2.
  \]
\end{proposition}
\begin{proof}
Denoting $\|\cdot\|$ the norm w.r.t. which $r_t^*$ is $\frac{1}{\mu}$-smooth,
we have for any $\u, \v \in \RR^p$, 
\begin{align}\label{eq:approx_reg}
r_t^*(\v) \leq \r_t(\u) 
+ \langle \nabla r_t^*(\u), \v - \u \rangle 
    + \frac{1}{2\mu}\|\u - \v\|^2.
\end{align}

Using $\v^t = (J_S^t)^* \alphav_S$, we obtain 
\begin{align*}
r_t^*\left((J_S^t)^*\alphav_S\right) 
&\leq r_t(\u^t) + \left\langle \nabla r_t^*(\u^t), (J_S^t)^*\alphav_S -
\u^t\right\rangle + \frac{1}{2\mu}\|(J_S^t)^*\alphav_S - \u^t \|_2^2 \\
&= \left\langle J_S^t \nabla r_t^*(\u^t), \alphav_S
\right\rangle + \frac{1}{2\mu}\|(J_S^t)^*\alphav_S\|_2^2 
-\frac{1}{\mu} \langle J_S^t \u, \alphav_S \rangle + \text{const w.r.t. }
\alphav_S.
\end{align*}
Plugging this quadratic upper-bound of 
$r^*_t\left((J_S^t)^*\alphav_S\right)$
back in \eqref{eq:dual_prox}, we arrive at the approximate
subproblem
\[
\min_{\alphav_S \in \RR^{m \times k}} 
\ell_S^*(\alphav_S) 
- \left\langle \f_S^t 
- J_S^t \left(\nabla r_t^*(\u^t) - \frac{1}{\mu} \u^t\right), \alphav_S \right\rangle
+ \frac{1}{2\mu} \|(J_S^t)^*\alphav_S\|^2.
\]
\end{proof}

\subsection{Dual of quadratic approximation of convex losses}
\label{app:dual_quad_logistic}

When using a quadratic-linear approximation of $h_i$,
the primal subproblem is, 
for a mini-batch $S = \{i_1, \ldots, i_m\} \subseteq [n]$, 
\begin{align*}
  \d_S^t & \coloneqq \argmin_{\d \in \RR^p} 
  q_S^t(-J_i^t \d) 
  - \langle \f_S^t, \d\rangle  
  + \frac{m}{2\gamma} \|\d\|_2^2 \\
  & = \argmin_{\d \in \RR^p} 
    \sum_{i \in S}
    \frac{1}{2} \langle \d, 
    (J_i^t)^* H_i^t J_i^t \d \rangle 
    -  \langle \d, (J_i^t)^* \g_i^t \rangle 
    + \frac{m}{2\gamma} \|\d\|_2^2 \\
  &  =  \argmin_{\d \in \RR^p}
  \frac{1}{2} \langle \d
  (J_S^t)^* H_S^t J_S^t \d \rangle 
  -  \langle \d, (J_S^t)^* \g_S^t \rangle 
  + \frac{m}{2\gamma} \|\d\|_2^2, \\
  & = \argmin_{\d \in \RR^p}
  a_S(-J_S^t \d)
  + \frac{m}{2\gamma} \|\d\|_2^2,
\end{align*}
where we used the shorthands
\begin{align*}
  a_S^t(\z) 
  & \coloneqq q_S^t(\z) - \langle \f_S^t, \z\rangle
  = \frac{1}{2} 
  \langle \z, H_S^t  \z \rangle 
  +  \langle \z, \g_S^t \rangle  = b_S^t(\z) + \langle \z, \g_S^t \rangle \\
  b_S^t(\z) & \coloneqq \sum_{j=1}^{m} b_{i_j}^t(\z_j) 
  = \sum_{j=1}^{m} \frac{1}{2} \langle \z_j, H_{i_j}^t \z_j \rangle 
  = \frac{1}{2}\langle \z, H_S^t \z\rangle, \\
  b_i^t(\z) & \coloneqq \frac{1}{2}\langle \z, H_i^t  \z \rangle.
\end{align*}
We recall that
\begin{align*}
    \f_i^t &\coloneqq f_i(\w^t) \\
    J_i^t &\coloneqq \partial f_i(\w^t) \\
    \g_i^t &\coloneqq \nabla \ell_i(\f_i^t) \\
    H_i^t &\coloneqq \nabla^2 \ell_i(\f_i^t)
\end{align*}
and similarly
\begin{align*}
    \f_S^t &\coloneqq f_S(\w^t) \coloneqq (f_{i_1}(\w^t), \ldots, f_{i_m}(\w^t)) \\
    \g_S^t &\coloneqq \nabla \ell_S(\f_S^t) \\
    H_S^t \u &\coloneqq \nabla^2 \ell_S(\f_S^t) \u = (\nabla^2 \ell_{i_1}(\f_{i_1}^t)
\u_1, \ldots, \nabla^2 \ell_{i_m}(\f_{i_m}^t)\u_m) \\
    J_S^t \d &\coloneqq \partial f_S(\w^t) \d \coloneqq
(\partial f_{i_1}(\w^t)\d, \ldots, \partial f_{i_m}(\w^t) \d) \\
    (J_S^t)^* \u &\coloneqq \partial f_S(\w^t)^* \u = \sum_{j=1}^m \partial f_{i_j}(\w^t)^*
\u_j.
\end{align*}

\paragraph{Stricly convex losses.}

We first present the result for strictly convex losses, in which case the convex
conjugate of interest is well-known. 
\begin{proposition}\label{prop:strict_cvx_dual}
  The prox-linear direction associated to a linear quadratic approximation of
  the objective
  \begin{align}\nonumber
   \d_S^t \coloneqq \argmin_{\d \in \RR^p} 
    \frac{1}{2} \langle \d, 
    (J_S^t)^* H_S^t J_S^t \d \rangle 
    -  \langle \d, (J_S^t)^* \g_S^t \rangle 
    + \frac{m}{2\gamma} \|\d\|_2^2,
  \end{align}
  with $H_i^t$ and so $H_S^t$ invertible, can be computed as $\d_S^t = \frac{\gamma}{m}(J_S^t)^* \alphav_S^t$,
  $\alphav_S^t = \g_S^t - \betav_S^t$ for 
  \begin{align*}
    \betav_S^t \coloneqq \argmin_{\substack{\betav \in \RR^{m\times k}}} & 
    \frac{1}{2} \langle \betav ,
    (H_S^t)^{-1} \betav \rangle
    + \frac{\gamma}{2m} \|(J_S^t)^* (\g_S^t - \betav)\|_2^2,
  \end{align*}
  where 
  \[
  (H_S^t)^{-1}\betav \coloneqq ((H_{i_1}^t)^{-1}\betav_1, \ldots, (H_{i_m}^t)^{-1}\betav_m).
  \]
\end{proposition}
\begin{proof}
  We use the same notations as in the beginning of the section. The convex
  conjugate of $a_S^t$ can be expressed in terms of the convex
  conjugate of $b_S^t$ as 
\begin{align*}
  (a_S^t)^*(\alphav) = (b_S^t)^*(\alphav - \g_S^t).
\end{align*}
The convex conjugate of $b_S^t$ itself can be expressed as 
\begin{align*}
  (b_S^t)^*(\betav) & = \sum_{j=1}^n (b_{i_j}^t)^*(\betav_j) 
  = \sum_{j=1}^{m} \frac{1}{2} \langle \betav_j, (H_{i_j}^t)^{-1} \betav_j\rangle 
  = \frac{1}{2} \langle  \betav, (H_S^t)^{-1} \betav\rangle,    
\end{align*}
using that $H_i^t$ is invertible such that $(b_i^t)^*(\betav) = \frac{1}{2}
\langle \betav, (H_i^t)^{-1} \betav\rangle$. The problem can then be solved as 
$
  \d_S^t = (J_S^t)^* \alphav_S^t
$
for 
\begin{align*}
  \alphav_S^t & \coloneqq \argmin_{\alphav \in \RR^{m \times k}}
  (a_S^t)^*(\alphav)
  + \frac{\gamma}{2m} \|(J_S^t)^* \alphav\|_2^2
  = \argmin_{\alphav \in \RR^{m \times k}}
  (b_S^t)^*(\alphav-\g_S^t)
  + \frac{\gamma}{2m} \|(J_S^t)^* \alphav\|_2^2
  = \g_S^t - \betav_S^t \\
  \betav_S^t & 
  \coloneqq \argmin_{\betav \in \RR^{m\times k}} (b_S^t)^*(-\betav) 
  + \frac{\gamma}{2m} \|(J_S^t)^* (\g_S^t - \betav)\|_2^2 
  = \argmin_{\betav \in \RR^{m\times k}}
  \frac{1}{2} \langle \betav,
  (H_S^t)^{-1} \betav \rangle
 + \frac{\gamma}{2m} \|(J_S^t)^* (\g_S^t - \betav)\|_2^2.
\end{align*}
\end{proof}
The approach above holds for example in the case of the squared loss. The
logistic loss on the other hand is not stricly convex, therefore its Hessian is
not invertible. We present below a generic derivation for any convex loss. We
then specialize the result for the logistic loss. 

\paragraph{Generic convex loss.}

In the generic case, we can tackle the computation of the dual of the
quadratic-linear approximation by using the pseudo-inverse of the Hessian as
stated in Proposition~\ref{prop:dual_approx_quad_gen_cvx}.
\begin{proposition}\label{prop:dual_approx_quad_gen_cvx}
  The prox-linear direction associated to a linear quadratic approximation of
  the objective
  \begin{align}\nonumber
   \d_S^t = \argmin_{\d \in \RR^p} 
    \frac{1}{2} \langle \d, 
    (J_S^t)^* H_S^t J_S^t \d \rangle 
    -  \langle \d, (J_S^t)^* \g_S^t \rangle 
    + \frac{m}{2\gamma} \|\d\|_2^2,
  \end{align}
  can be computed as $\d_S^t = \frac{\gamma}{m}(J_S^t)^* \alphav_S^t$,
  $\alphav_S^t = \g_S^t - \betav_S^t$ for 
  \begin{align*}
    \betav_S^t = \argmin_{\betav \in \RR^{m\times k}} & 
    \frac{1}{2} \langle \betav ,
    (H_S^t)^{\dagger} \betav \rangle
    + \frac{\gamma}{2m} \|(J_S^t)^* (\g_S^t - \betav)\|_2^2 \\
    \mbox{s.t.} \quad & (I - H_S^t (H_S^t)^\dagger) \betav = {\bm 0},
  \end{align*}
  where $(H_i^t)^\dagger$ denotes the pseudo inverse of $H_i^t$ and 
  \[
  (H_S^t)^{\dagger}\betav = ((H_{i_1}^t)^{\dagger}\betav_1, \ldots, (H_{i_m}^t)^{\dagger}\betav_m).
  \]
  \end{proposition}
\begin{proof}
  The proof follows the same reasoning as in
  Proposition~\ref{prop:strict_cvx_dual} except that for generic convex loss,
  the convex conjugate of $b_i^t$ is given by
  Lemma~\ref{prop:generic_dual_semi-definite_quadratic} as 
  \[
    (b_i^t)^*(\betav) = \begin{cases}
      \frac{1}{2} \langle \betav, (H_i^t)^{\dagger} \betav \rangle  
      & \mbox{if} \ (\idm - H_i^t(H_i^t)^\dagger) \betav = 0 \\
      + \infty 
      & \ \mbox{otherwise}.
    \end{cases}
  \]
  The result follows using that $H_S^t (H_S^t)^\dagger\betav = (H_{i_1}^t
  (H_{i_1}^t)^\dagger \betav_1, \ldots, H_{i_m}^t (H_{i_m}^t)^\dagger \betav_m)$.
\end{proof}

\begin{lemma}\label{prop:generic_dual_semi-definite_quadratic}
  Let $q(\w) \coloneqq \frac{1}{2}\langle \w, A \w\rangle$,
  where $A\succeq 0$, $A \in \RR^{k\times k}$.
  The convex conjugate of $q$ is 
  \[
    q^*(\v) = 
    \begin{cases}
       \frac{1}{2} \langle \v, A^\dagger\v \rangle & \mbox{if} \ AA^\dagger\v = \v\\
       +\infty & \mbox{otherwise}
    \end{cases}
  \] 
  where $A^\dagger$ denotes the pseudo-inverse of $A$.
\end{lemma}
\begin{proof}
  Denote $P = I - A^\dagger A $ the projection on the null-space of $A$. Note
  that as $A$ is symmetric, we have $P = I - AA^\dagger $. Since $A \succeq 0$, $q$ is convex and its
  conjugate is defined as 
  \[
    q^*(\v) = \sup_{\w \in \RR^k} \langle\v, \w \rangle 
    - \frac{1}{2} \langle \w, A \w\rangle.
  \]
  If $P\v \neq 0$, then by considering $\w(t) = t P\v$ for $t \in \RR$, we have
  $
    \langle \v, \w(t)\rangle 
    - \frac{1}{2} \langle \w(t), A \w(t)\rangle = t \|P\v\|_2^2 
  $
  which tends to $+\infty$ for $t \rightarrow + \infty$. Hence, $q^*(\v)=
  +\infty$ if $P\v \neq 0$.
  
  If $\v =AA^\dagger \v$. The convex conjugate then amounts to solve
  \[
    \sup_{\w \in \RR^d} \langle \w, AA^\dagger \v \rangle 
    - \frac{1}{2} \langle \w, A \w\rangle.
  \]
  The solution of this problem is given by $\w^\star$ such that 
  $
    AA^\dagger \v = A \w^\star,
  $
  hence $\w^\star = A^\dagger \v$ is a solution and the convex conjugate is then
  \begin{align*}
    q^*(\v) & = \langle \v, {A^\dagger}^* AA^\dagger \v \rangle
    - \frac{1}{2} \langle \v, {A^\dagger}^* A A^\dagger \v \rangle 
    = \langle \v, A^\dagger AA^\dagger \v \rangle
    - \frac{1}{2} \langle \v, A^\dagger A A^\dagger \v\rangle 
    = \frac{1}{2} \langle \v, A^\dagger \v \rangle,
  \end{align*}
  where we used that $A^\dagger$ is symmetric since $A$ is
  symmetric and we used the identity $A^\dagger A A^\dagger =
  A^\dagger$.
\end{proof}

\paragraph{Logistic loss.}

We now derive the conjugate of the quadratic approximation in the case of the
logistic loss.
\begin{proposition}\label{prop:dual_approx_quad_logloss}
  Consider the logistic loss
  \[
    \ell(\f) = - \langle \y , \f\rangle  + \phi(\f),
  \]
  for $\y \in \{0, 1\}^k$, $\y^\top \ones_k = 1$ and $\phi(\f) = \log(
  \exp(\f)^\top \ones_k)$ for $\f \in \RR^k$, where $\exp$ is applied
  element-wise. Consider the quadratic approximation of the logistic loss at a point $\f$
  given by 
  \[
  q(\ell, \f)(\v) \coloneqq
    -  \langle \y - \nabla \phi(\f), \v \rangle 
    + \frac{1}{2}\v^\top \nabla^2 \phi(\f)\v,
  \] 
  where
  $
    \nabla \phi(\f) = \sigma(\f), 
    \nabla^2 \phi(\f) = \diag\left(\sigma(\f)\right) 
    - \sigma(\f) \sigma(\f)^\top,  
    \sigma(\f) \coloneqq \mathrm{softmax}(\f) 
    = {\exp(\f)}/(\exp(\f)^\top \ones_k).
  $
  Its convex conjugate is, for $\betav \coloneqq \alphav - \nabla \ell(\f)$
  and 
  $D \coloneqq \diag(\sigma(\f))$,
  \begin{align*}
  q(\ell, \f)^*(\alphav) = \begin{cases}
      \frac{1}{2} \langle \betav,  
      D^{-1}
      \betav \rangle 
      & \mbox{if} \ \betav^\top \ones_k = 0 \\
      +\infty & \mbox{otherwise}.
    \end{cases}
  \end{align*}
\end{proposition}
\begin{proof}
  The convex conjugate reads
  \[
      q(\ell, \f)^*(\alphav) = h^*(\alphav + \y - \nabla \phi(\f)),
  \]
  where
  \[
    h^*(\betav) = \sup_{\v  \in \RR^k} \betav^\top \v 
    - \frac{1}{2} \langle \v,  \nabla^2 \phi(\f) \v \rangle.
  \]
  Note that $\nabla^2 \phi(\f) \ones_k = 0$. In the following, denote $\betav =
  \alphav + \y - \nabla \phi(\f)$ and consider computing $h^*(\betav)$. Note
  that $\betav^\top \ones_k = \alphav^\top \ones_k$ since $\y^\top \ones_k = \nabla
  \phi(\f)^\top \ones_k = 1$.
  
  If $\betav^\top \ones_k \neq 0$, that is $\alphav^\top \ones_k \neq 0$, then by
  considering $\v(t) = t \ones_k \ones_k^\top\betav$, we have
  \[
    \betav^\top \v(t)
    - \frac{1}{2} \langle \v(t), \nabla^2 \phi(\f) \v(t) \rangle
    = t (\ones_k^\top \betav)^2 \underset{t\rightarrow +\infty}{\longrightarrow} + \infty
  \]
  so $h^*(\betav) = +\infty$ and $q(\ell, \f)^*(\alphav) = +\infty$.

  Consider now $\betav^\top \ones_k = 0$, that is $\alphav^\top \ones_k  = 0$
  and $\v^\star = D^{-1} \betav$ for $D = \diag\left(\frac{\exp(\f)}{\exp(\f)^\top
  \ones_k}\right)$. We have then 
  \begin{align*}
    \nabla^2\phi(\f) \v^\star
    = DD^{-1}\betav 
    - \frac{\exp(\f)}{\exp(\f)^\top \ones_k} \ones_k^\top \betav = \betav.
  \end{align*}
  Hence $\v^\star$ satisfies the first-order conditions of the problem defining
  $h^*(\betav)$, so it is a solution of that problem. The expression of the
  convex conjugate follows.
\end{proof}
Using this expression of the conjugate, we obtain \eqref{eq:dual_quad_logistic}
as formally stated below.
\begin{corollary}
  The prox-linear direction associated to a linear quadratic approximation of
  the objective
  \begin{align}\nonumber
   \d_S^t = \argmin_{\d \in \RR^p} 
    \frac{1}{2} \langle \d, 
    (J_S^t)^* H_S^t J_S^t \d \rangle 
    -  \langle \d, (J_S^t)^* \g_i^t \rangle 
    + \frac{m}{2\gamma} \|\d\|_2^2,
  \end{align}
  for $\ell_i$ the logistic loss,  
  can be computed as $\d_S^t = \frac{\gamma}{m}(J_S^t)^* \alphav_S^t$,
  $\alphav_S^t = \g_S^t - \betav_S^t$ for 
  \begin{align*}
    \betav_S^t = \argmin_{\betav \in \RR^{m\times k}} & 
    \frac{1}{2} \langle \betav ,
    (D_S^t)^{-1} \betav \rangle
    + \frac{\gamma}{2m} \|(J_S^t)^* (\g_S^t - \betav)\|_2^2, \\
    \mbox{s.t.} \quad & \ones_k^\top \betav_i = {\bm 0} \ \mbox{for} \ i \in \{1, \ldots m\}
  \end{align*}
  for $(D_S^t)^{-1}\betav = (\betav_1 /\sigma(\f_{i_1}^t), \ldots, \betav_m / \sigma(\f_{i_m}^t))$.
\end{corollary}

\subsection{Conjugate gradient method for quadratic approximations}\label{app:ssec:quad_dual_resolution}

\begin{proposition}\label{prop:ecqp_via_precond}
  Consider a quadratic problem under linear constraints of the form 
  \begin{align}\label{eq:ecqp}
    \min_{\betav\in \RR^d} \ & \frac{1}{2} \langle \betav,  Q \betav \rangle - \langle \betav, \c \rangle \\
    \mbox{s.t.} \ & (I - P) \betav= 0 \nonumber
  \end{align}
  for $Q$ semi-definite positive and $P$ an orthonormal projector, that is, $P =
  P^*$ and $PP = P$. Assume that $\betav\mapsto\frac{1}{2} \langle \betav,  Q
  \betav \rangle - \langle \betav, \c\rangle $ is bounded below.

  Any convergent first-order optimization algorithm applied to the unconstrained problem
  \begin{align}\label{eq:ecqp_unconstrained}
    \min_{\betav\in \RR^d} \ & \frac{1}{2} \langle \betav, P Q P\betav \rangle 
    - \langle \betav, P \c \rangle
  \end{align}
  and initialized at $\betav^0 = \zeros$ 
  converges to a solution of~\eqref{eq:ecqp}.
\end{proposition}
\begin{proof}
  First, note that problem~\eqref{eq:ecqp} is necessarily feasible as its
  constraints are satisfied for $\betav=\zeros$. 
  Moreover problem~\eqref{eq:ecqp} admits
  a minimizer since it is the minimization of a convex quadratic bounded below
  on a subspace of $\RR^d$. 

  Similarly, if $\betav\mapsto\frac{1}{2} \langle \betav,  Q \betav\rangle -
  \langle \betav, \c \rangle$ is bounded below, then necessarily $\c$ does not
  belong to the null space of $Q$, otherwise taking $\betav = t\c$,
  $t\rightarrow + \infty$, would lead to $-\infty$. Since the null space and the
  image of $Q$ are orthonormal spaces, $\c$ belongs to the image of $Q$, so that
  there exists $\d \in \RR^d$ satisfying $\c= Q\d$. The objective
  of~\eqref{eq:ecqp_unconstrained} can then be factorized as $\frac{1}{2}
  \langle \betav,  P Q P\betav \rangle - \langle \betav, P \c \rangle   =
  \frac{1}{2}\langle (P\betav-\d),  Q(P\betav -\d)\rangle  -
  \frac{1}{2}\langle \d, Q \d \rangle$. Hence, it is a convex quadratic bounded
  below, so it admits a minimizer.
  
  A point $\betav^{\star}$ is optimal for~\eqref{eq:ecqp} if there exists
  $\lambdav^{\star} \in \RR^d$ such that 
  \[
    P Q P \betav^{\star} - P\c + P\lambdav^{\star} - \lambdav^{\star} = \zeros, 
    \quad P\betav^{\star} = \betav^{\star}.
  \]
  In comparison, a convergent first-order optimization algorithm applied
  to~\eqref{eq:ecqp_unconstrained} converges to a point $\hat \betav$
  satisfying the first order optimality conditions
  of~\eqref{eq:ecqp_unconstrained}, that is,
  \[
    P QP \hat \betav  - P\c = \zeros.
  \]
  The iterates of any first-order optimization algorithm are built such that 
  \[
    \betav^k \in \mathrm{Span}(
    \betav^0, \nabla f(\betav^0),
    \betav^1, \nabla f(\betav^1), 
    \ldots,
    \betav^{k-1}, \nabla f(\betav^{k-1})
    ),
  \]
  where $f(\betav) = \frac{1}{2} \langle \betav,  P QP\betav \rangle  -
  \langle \betav,  P\c \rangle$. Denote $\mathcal{C} \coloneqq \{\betav :
  P\betav =\betav\}$, that is a subspace of $\RR^d$. We have that for any
  $\betav \in \RR^d$, $\nabla f(\betav) = PQP \betav - P\c \in \mathcal{C}$
  since $PP=P$. If $\betav^0 = \zeros$, then $\betav^0 \in \mathcal{C}$ and by
  induction we have that $\betav^k \in \mathcal{C}$. Therefore, $\hat \betav =
  \lim_{k \rightarrow +\infty}\betav^k$ satisfies $P\hat \betav =
  \lim_{k\rightarrow + \infty} P \betav^k = \lim_{k\rightarrow + \infty}
  \betav^k = \hat \betav$. Therefore, $\hat \betav$ satisfies
  \[
    P QP \hat \betav  - P\c = 0, \quad P\hat \betav = \hat \betav.
  \]
  It is therefore a solution of~\eqref{eq:ecqp} with associated
  $\lambdav^{\star} = \zeros$.
\end{proof}

The above proposition can directly be applied to the problems associated to
$\betav_S^t$ in Proposition~\ref{prop:dual_approx_quad_gen_cvx}
and~\ref{prop:dual_approx_quad_logloss} they are convex quadratic problem
unbounded below.

\begin{corollary}\label{cor:dual_approx_quad_gen_cvx}
  The prox-linear direction associated to a linear quadratic approximation of
  the objective
  \begin{align}\nonumber
   \d_S^t = \argmin_{\d \in \RR^p} 
    \frac{1}{2} \langle \d, 
    (J_S^t)^* H_S^t J_S^t \d \rangle 
    -  \langle \d, (J_S^t)^* \g_S^t \rangle 
    + \frac{m}{2\gamma} \|\d\|_2^2,
  \end{align}
  can be computed as $\d_S^t = \frac{\gamma}{m}(J_S^t)^* \alphav_S^t$,
  $\alphav_S^t = \g_S^t - \betav_S^t$ for $\betav_S^t$ solving
  \begin{align*}
    \betav_S^t = \argmin_{\betav \in \RR^{m\times k}} & 
    \frac{1}{2} \langle \betav , P 
    (H_S^t)^\dagger P  \betav \rangle
    + \frac{\gamma}{2m} \|(J_S^t)^* (\g_S^t - P\betav)\|_2^2, \
  \end{align*}
  for $P = H_S^t (H_S^t)^\dagger$ and $\betav_S^t$ computed by a conjugate
  gradient method initialized at ${\bm 0}$.
\end{corollary}

\begin{corollary}
  The prox-linear direction associated to a linear quadratic approximation of
  the objective
  \begin{align}\nonumber
   \d_S^t = \argmin_{\d \in \RR^p} 
    \frac{1}{2} \langle \d, 
    (J_S^t)^* H_S^t J_S^t \d \rangle 
    -  \langle \d, (J_S^t)^* \g_S^t \rangle 
    + \frac{m}{2\gamma} \|\d\|_2^2,
  \end{align}
  for $\ell_i$ the logistic loss, can be computed as $\d_S^t = \frac{\gamma}{m}(J_S^t)^* \alphav_S^t$,
  $\alphav_S^t = \g_S^t - \betav_S^t$ for $\betav_S^t$ solving
  \begin{align*}
    \betav_S^t = \argmin_{\betav \in \RR^{m\times k}} & 
    \frac{1}{2} \langle \betav , P 
    (D_S^t)^{-1} P  \betav \rangle
    + \frac{\gamma}{2m} \|(J_S^t)^* (\g_S^t - P\betav)\|_2^2, \
  \end{align*}
  for $P \betav = (\Pi_k \betav_1, \ldots, \Pi_k \betav_m)$, $\Pi_k =
  \frac{1}{n}\ones_k\ones_k^\top$ and $\betav_S^t$ computed by a conjugate
  gradient method initialized at ${\bm 0}$.
\end{corollary}

In practice, the matrix $D^{-1}$ may also be ill-conditioned as it consists of
the diagonal of the reciprocal of the softmax, whose values may be close to 0. To
avoid this problem, we precondition the problem associated with $\betav_S^t$ by
$D^{1/2}$. 

\subsection{Prox-linear directions define critical points}
We recall here that prox-linear directions define critical points. This will
help us refine the results about descent directions. We state it for individual
functions for simplicity. The result can readily be generalized for
mini-batches. Nullity of $\d(\gamma \ell_i, f_i)(\w^t)$ can also be linked to
the fact that $\w^t$ is close to a stationary point as shown
by~\citet{drusvyatskiy2019efficiency}.
\begin{proposition}\label{prop:critical_point} If $\w^\star$ is a minimum of
  $\ell_i \circ f_i$ then the prox-linear direction~\eqref{eq:subproblem_i}
  $\d(\gamma \ell_i, f_i)(\w^\star)$ or its quadratic
  approximation~\eqref{eq:quadratic_subproblem_minibatch} $\d(\gamma q_i^\star,
  f_i)(\w^\star)$, with $q_i^\star$ the quadratic approximation of $\ell_i$
  around $f_i(\w^\star)$, is zero. On the other hand, if $\d(\gamma q_i^\star,
  f_i)(\w^\star)= {\bm 0}$, then $\nabla (\ell_i\circ f_i)(\w^\star) = {\bm 0}$,
  that is $\w^\star$ is a critical point. 
\end{proposition}
\begin{proof}
  Suppose $\w^\star$ is a minimum of $\ell_i \circ f_i$. Denote $F(\d) =
  \ell_i(f_i(\w^\star)-\partial f_i(\d)) + \gamma \|\d\|_2^2/2$. Since
  $\w^\star$ is the minimum of $\ell_i \circ f_i$, ${\bm 0}$ is the minimizer of $F$, hence
  since $\d(\gamma \ell_i, f_i)(\w^\star)$ is defined as the minimizer of $F$ it
  must be ${\bm 0}$. For $\d(\gamma q_i^\star, f_i)(\w^\star)$, then $\nabla
  (\ell_i\circ f_i)(\w^\star) = {\bm 0}$ so the direction reduces to compute
  $\d(\gamma q_i^\star, f_i)(\w^\star) = \argmin_{\d \in \RR^p} \frac{1}{2}
  \langle\d, Q\d\rangle +\frac{1}{2\gamma}\|\d\|_2 ={\bm 0}$ for $Q = \partial
  f_i(\w^\star)^* \nabla^2 \ell_i(f_i(\w^\star)) \partial f_i(\w^\star)\succeq
  0$. On the other hand, we have for any $\w^\star$, $\d(\gamma q_i^\star,
  f_i)(\w^\star) = (\gamma^{-1}\idm + Q)^{-1} \nabla (\ell_i \circ f_i)(\w^*)$,
  so $\d(\gamma q_i^\star, f_i) ={\bm 0} \iff \nabla (\ell_i\circ f_i)(\w^\star)
  = {\bm 0}$.
\end{proof}

\subsection{Proof of Proposition \ref{prop:descent_direction} (descent
direction, exact case)}\label{app:proof_descent_dir} 

We show a slightly stronger result, namely that, unless $\d(\gamma q_S^t,
\f_S)(\w^t) ={\bm 0}$, that is, $\w^t$ is a critical point of $\ell_S \circ
\f_S$ as shown in Proposition~\ref{prop:critical_point}, the direction
$\d(\gamma q_S^t, \f_S)(\w^t)$ satisfies $\langle\d(\gamma q_S^t, \f_S)(\w^t),
\nabla (\ell_S\circ \f_S)(\w^t)\rangle > 0$. The result claims in the main text
holds naturally for $\d(\gamma q_S^t, \f_S)(\w^t)= {\bm 0}$.

Denote 
\[
  F_S^t(\d) \coloneqq \ell_S(f_S(\w^t) - \partial f_S(\w^t) \d ) +
\frac{m}{2\gamma}\|\d\|_2^2
\]
and $\d^\star \coloneqq \d(\gamma \ell_S, f_S)(\w)
= \argmin_{\d \in \RR^p} F_S^t(\d) $. Since $F_S^t$ is strongly convex and
assuming $\d^\star \neq 0$, we have $F_S^t(\d^\star) > F_S^t({\bm 0}) +
\nabla F_S^t({\bm 0})^\top \d^\star$, that is, $\nabla F_S^t({\bm 0})^\top
\d^\star < F_S^t(\d^\star) - F_S^t({\bm 0}) < 0$ since $\d^\star
=\argmin_{\d \in \RR^p} F_S^t(\d)$. Note that $\nabla F_S^t({\bm 0}) = -\partial
f_S(\w^t)^*\nabla \ell_S(f_S(\w^t)) = - \nabla (\ell_S \circ f_S)(\w^t)$. Hence,
$\langle \nabla (\ell_S \circ f_S)(\w^t) , -\d^\star  \rangle < 0$.

For the quadratic case, a similar reasoning applies. Denote now 
\[
  G_S^t(\d)
\coloneqq
q_S^t(\f_S^t - \partial f_S(\w^t)\d) + \frac{m}{2\gamma}\|\d\|_2^2.
\]
We have that $G_S^t$ is strongly convex with minimizer $\d(\gamma q_S,
f_S)(\w^t)$. Moreover, we have that $\nabla G_S^t({\bm 0}) = -\partial
f_S(\w^t)^*\nabla \ell_S(f_S(\w^t)) = - \nabla (\ell_S \circ f_S)(\w^t)$. Hence,
as above, assuming $\d(\gamma q_S, \f_S)(\w^t) \neq 0$, we get that $\langle
\nabla (\ell_S \circ f_S)(\w^t) , -\d(\gamma q_S, f_S)(\w^t) \rangle < 0$.

\subsection{Proof of Proposition \ref{prop:descent_dir_cg_primal} (descent
direction, inexact case)}
~
\paragraph{Primal case.}

\begin{proposition}
  Denote $\d_S^{t, \tau}$ the approximate solution of the following problem
  computed by a conjugate gradient method after $\tau$ iterations,
  \begin{align*}
    \argmin_{\d \in \RR^p} \
    \frac{1}{2} \langle \d, J^* H J \d \rangle
    - \langle \d, J^* \g \rangle
    + \frac{m}{2\gamma}\|\d\|_2^2
  \end{align*}
  for $J \coloneqq \partial \f_S(\w^t)$, $H \coloneqq \nabla^2 \ell_S(\f^t_i)$,
  $\g \coloneqq \nabla \ell_S(\f_i^t)$, $\f_S^t \coloneqq f_S(\w^t)$, $\gamma >
  0$, $m \coloneqq |S|$. Then $-\d_S^{t, \tau}$ is a descent direction for $\ell_S\circ
  f_S$ at $\w^t$. 
\end{proposition}
\begin{proof}
  The problem at hand is a convex quadratic of the form 
$
\min_{\d \in \RR^p} \frac{1}{2}\langle  \d, Q \d \rangle  
- \langle  \c, \d \rangle 
$
for $Q \coloneqq J^* H J + (m/\gamma)\idm $ and $\c \coloneqq J^* \g$. Using
Lemma~\ref{lem:cg_positive}, we have that the $\tau$\textsuperscript{th} iterate
of a conjugate gradient method applied to the above problem satisfies $\langle
\d_S^{t, \tau}, \c\rangle \geq 0$. Since $\c = \nabla h_S(\w^t)$, the result
follows. 
\end{proof}

\paragraph{Dual case.}
\begin{proposition}
  Consider computing the prox-linear direction
  \begin{align*}
    \d_S^t = \argmin_{\d \in \RR^p} \
    \frac{1}{2} \langle \d, J^* H J \d \rangle
    - \langle \d, J^* \g \rangle
    + \frac{m}{2\gamma}\|\d\|_2^2
  \end{align*}
  for $J = \partial \f_S(\w^t)$, $H = \nabla^2 \ell_S(\f^t_i)$, $\g = \nabla
  \ell_S(\f_S^t)$, $\f_S^t = f_S(\w^t)$, $\gamma > 0$, via its dual formulation
  as $\d_S^t = \frac{\gamma}{m} J^*\alphav_S^t$, $\alphav_S^t = \g- \betav_S^t$
  for
  \begin{align}\label{eq:dual_cg_gen}
    \betav_S^t = \argmin_{\betav \in \RR^k} \
    \frac{1}{2} \langle \betav, P\left(H^\dagger
    + \frac{\gamma}{m} JJ^*\right) P \betav  \rangle
    - \frac{\gamma}{\beta} \langle \betav, P JJ^* \g \rangle,
  \end{align}
  and $P = HH^\dagger$ as presented in
  Corollary~\ref{cor:dual_approx_quad_gen_cvx}. Let $\betav_S^{t, \tau}$ be the
  $\tau$\textsuperscript{th} iteration of a conjugate gradient method applied
  to~\eqref{eq:dual_cg_gen} with associated primal direction $\d_S^{t, \tau} =
  \frac{\gamma}{m} ( J^* \g - J^* \betav_S^{t, \tau} )$. We have
  $
    \langle {\d_S^{t, \tau}}, \nabla (\ell_i \circ f_i)(\w^t) \rangle \geq 0
  $
  so that $-\d_S^{t, \tau}$ is a descent direction for $\ell_S \circ f_S$ at
  $\w^t$.
\end{proposition}
\begin{proof}
  For simplicity denote $\tilde \gamma = \gamma/m$. Consider
  problem~\eqref{eq:dual_cg_gen} in a canonical form
  \[
    \min_{\betav \in \RR^k} \ 
    \frac{1}{2} \langle \betav, Q \betav  \rangle
    - \langle \betav, \c \rangle,
  \]
  for $Q = P(H^\dagger + \tilde \gamma JJ^*)P$, $\c = \tilde \gamma P JJ^* \g$. Consider
  $\betav^\tau$ the $\tau$\textsuperscript{th} iteration of a conjugate gradient
  method applied to the above problem whose iterations are presented in
  Lemma~\ref{lem:cg_positive}. Our goal is to show that for any $\tau \geq 0$,
  $\langle {\d^\tau},  \nabla (\ell_S \circ f_S)(\w^t)\rangle \geq 0$ for
  $\d^\tau =  \tilde \gamma ( J^* \g - J^* \betav^\tau )$, which reads
  \[
    \langle \betav^\tau, JJ^* \g \rangle \leq \langle \g, JJ^*\g \rangle.
  \]
  Note that given the forms of $Q$ and $\c$ above, the iterates of a conjugate
  gradient method satisfy $\betav^\tau = P \betav^\tau$ with $P$ an orthonormal
  projector satisfying $P = P^*$. The above condition is then equivalent to
  $\tilde \gamma^{-1}\langle \betav^\tau, \c \rangle = \langle \betav^\tau, PJJ^* \g
  \rangle  \leq \langle \g, JJ^*\g \rangle.$ We proceed by contradiction and
  assume there exists $\tau_0 \in \{0, 1, \ldots, +\infty\}$ such that 
  \begin{equation}
    \label{eq:contradiction_gen}
    \langle \betav^{\tau_0}, JJ^* \g \rangle > \langle \g, JJ^*\g \rangle.
  \end{equation}
  Recall that, with the notations of Lemma~\ref{lem:cg_positive}, for any $\tau
  \geq 0$, $\tilde \gamma \langle \p^\tau, JJ^* \g \rangle = \langle \p^\tau,
  \c\rangle \geq 0$. If~\eqref{eq:contradiction_gen} is true, then, for any $\tau >
  \tau_0$, since $\betav^\tau = \betav^{\tau_0} + \sum_{s=\tau_0+1}^\tau a_s
  \p^{s-1}$, $a_s \geq 0$ and $\tilde \gamma \geq 0$, we have that
  $
    \langle \betav^{\tau}, JJ^* \g \rangle > \langle \g, JJ^*\g \rangle.
  $
  Taking $\tau \rightarrow +\infty$, we have then 
  \begin{align*}
    0 \leq \lim_{t\rightarrow + \infty}\langle J^* (\betav^\tau - \g), J^* \g \rangle 
    = \langle J^* (\betav^\star - \g), J^* \g \rangle 
    & = \langle -\d^\star, J^* \g\rangle  ,
  \end{align*}
  for $\d^\star$ the prox-linear direction. This contradicts
  Proposition~\ref{prop:descent_direction} where in the proof we showed that
  the prox-linear direction satisfies $ \langle \d^\star, J^* \g\rangle > 0$, unless
  $\d^\star={\bm 0}$, in which the case, the claim holds trivially. Hence, we
  have shown the claim, i.e., that for any $\tau \geq 0$, $\langle \betav^\tau,
  JJ^* \g \rangle \leq  \langle \g, JJ^*\g \rangle$. Therefore the output primal
  direction $\d^\tau$ satisfies $\langle {\d^\tau},  \nabla (\ell_S \circ
  f_S)(\w^t)\rangle \geq 0$.
\end{proof}

\begin{lemma}\label{lem:cg_positive}
  Consider the iterations of a conjugate gradient method for solving 
  $\min_{\x\in \RR^d} \frac{1}{2} \langle \x,  Q \x\rangle - \langle \c,  \x
  \rangle$, i.e., starting from 
  $\x^0 \coloneqq {\bm 0}$ 
  and
  $\r^0 \coloneqq \p^0 \coloneqq \c$,
    \begin{align*}
      a^\tau &\coloneqq \langle \r^{\tau-1},  \r^{\tau-1} \rangle 
      / \langle \p^{\tau-1}, Q \p^{\tau-1}\rangle \\ 
      \x^\tau &\coloneqq \x^{\tau-1} + a^\tau \p^{\tau-1} \\
      \r^\tau &\coloneqq \r^{\tau-1} - a^\tau Q \p^{\tau-1} \\
      b^\tau &\coloneqq \langle \r^\tau,  \r^\tau \rangle/
       \langle \r^{\tau-1}, \r^{\tau-1}\rangle\\
      \p^\tau &\coloneqq \r^\tau + b^\tau \p^{\tau-1}.
    \end{align*}
  Then, for any $\tau \geq 0$, we have $\langle \p^\tau, \c\rangle \geq 0$ and
    \[
      \langle \x^\tau, \c \rangle \geq 0.
    \]
  \end{lemma}
  \begin{proof}
    Since $\x^\tau = \x^0 + \sum_{s=1}^\tau a^s \p^{s-1}$ with $a_s \geq 0$ for all $s$ and $\x^0 = {\bm 0}$, it suffices to show that $\langle \p^\tau, \c \rangle \geq 0$ for all $\tau \geq 0$. For $\tau = 0$, we have $\langle \p^0, \c\rangle  = \|\c\|^2_2 \geq 0$. Assume $\langle \p^{\tau-1}, \c \rangle \geq 0$ for $\tau \geq 1$, then, as $b^\tau \geq 0$,
    \[
      \langle \p^\tau, \c\rangle  = \langle \r^\tau, \c \rangle + b^\tau \langle \p^{\tau-1}, \c\rangle \geq \langle \r^\tau, \b \rangle = \langle \r^\tau, \r^0 \rangle = 0,
    \]
    where the last equality comes from the orthogonality of the residuals in a
    conjugate gradient method. The proof has then be shown by induction.
  \end{proof}

\section{Computational complexities} \label{app:comput_cost}
The computational complexities of the primal and dual formulation can be
formally compared in the case of the squared loss or the quadratic
approximation of the loss. We consider for simplicity here the formulation of
the prox-linear step for squared losses, whose primal formulation reads, for a
mini-batch $S = \{i_1, \ldots, i_m\}$,
\begin{equation}\label{eq:cpt_cplxity_primal_ls}
  \d_S^t = \argmin_{\d \in \RR^p} \, 
  \frac{1}{2} \langle \d, \left((J_S^t)^* J_S^t + (m/\gamma)\idm\right)\d\rangle
  - \langle (J_S^t)^* \g_S^t, \d\rangle,
\end{equation}
with associated dual formulation,
\begin{equation}\label{eq:cpt_cplxity_dual_ls}
  \alphav_S^t = \argmin_{\alphav \in \RR^p} \, 
  \frac{1}{2} \langle \alphav, \left(J_S^t (J_S^t)^* + (\gamma/m) \idm\right) \alphav\rangle
  - \langle \g_S^t, \alphav\rangle, \quad \d_S^t = (J_S^t)^* \alphav_S^t,
\end{equation}
and associated shortcuts,
\begin{align*}
  \f_S^t & = \f_S(\w^t) = (\f_{i_1}(\w^t), \ldots, \f_{i_m}(\w^t)) \in \RR^{m \times k}, \\
  \g_S^t & = \nabla \ell_S(\f_S^t) 
  = (\nabla \ell_{i_1}(\f_1^t), \ldots, \nabla \ell_{i_m}(\f_m^t)) \in \RR^{m \times k}, \\
  J_S^t \u & = \partial \f_S(\w^t)(\u) 
  = (J_{i_1}^t \u, \ldots, J_{i_m}^t \u), 
  \quad J_i^t \u = \partial \f_i(\w^t)\u, 
  \quad \mbox{for} \ \u \in \RR^p, \\
  (J_S^t)^* \v & = (\partial \f_S(\w^t))^* \v 
  = \sum_{j=1}^m (J_{i_j}^t)^* \v_j,
  \quad \mbox{for} \ \v = (\v_1, \ldots, \v_m) \in \RR^{m \times k}.
\end{align*}
For each formulation, we need to consider (i) the cost of each iteration of the
inner solver, (ii) the condition number of each subproblem.

\paragraph{Computational cost of running $\tau$ iterations of the inner solver.}
Consider a conjugate gradient (CG) applied to solve $\min_{\x \in \RR^d}
\frac{1}{2} \langle \x, Q \x\rangle - \langle \x, \c\rangle$ with $Q$ positive
definite as recalled in Lemma~\ref{lem:cg_positive}. Each iteration of CG requires 
\begin{enumerate}
  \item 1 call to the linear operator $\x \mapsto Q \x$ to compute $Q
  \p^{\tau-1}$; if this linear call amounts to a matrix-vector computation, it
  costs $O(d^2)$ elementary computations, 
  \item 3 inner products computations in $\RR^d$, each at a cost of $O(d)$
  elementary computations, to compute $\langle \r^{\tau-1}, \r^{\tau -1}
  \rangle$, $\langle \r^\tau, \r^\tau \rangle$, $\langle \p^{\tau -1}, Q
  \p^{\tau-1}\rangle$,
  \item 3 additions in $\RR^d$, each at a cost of $O(d)$ elementary
  computations, to compute $\x^\tau, \r^\tau, \p^\tau$. 
\end{enumerate}

For the primal formulation \eqref{eq:cpt_cplxity_primal_ls}, computing the
linear operator call $\x \mapsto Q_{\mathrm{primal}} \x$, with
$Q_{\mathrm{primal}} \coloneqq (J_S^t)^* J_S^t + (m/\gamma)\idm$ amounts to one
call to the JVP $J_S^t$ and one call to the VJP $(J_S^t)^*$. For the dual
formulation, computing the linear operator call $\x \mapsto Q_{\mathrm{dual}}
\x$, with $Q_{\mathrm{dual}} \coloneqq J_S^t (J_S^t)^* + (\gamma/m)\idm$ amounts
to one call to the VJP $(J_S^t)^*$ and one call to the JVP $J_S^t$. Hence, the
computation of the linear operator is the same in both formulations. 

Other computational costs differ whether the primal or the dual formulation is
considered. For the primal formulation \eqref{eq:cpt_cplxity_primal_ls}, the
variables are of the dimension of the parameters, that is all other computations
incur a cost of $O(p)$ elementary computations. On the other hand, for the dual
formulation~\eqref{eq:cpt_cplxity_dual_ls}, the variables are of dimension the
number of samples times the output dimension, that is all other computations
incur a cost of $O(m k)$ in this case. For small batches, we retrieve that the
dual formulation can be advantageous. 

Finally, each formulation requires one additional call to the VJP $(J_S^t)^*$.
For the primal formulation, this call is necessary to compute $(J_S^t)^* \g_S^t$.
For the dual formulation, this call is necessary to map back the dual solution
to the primal solution, i.e., computing $(J_S^t)^* \alphav_S^t$. 

To summarize, the computational cost of running $\tau$ iterations of CG to
compute the prox-lienar update is 
\begin{align*}
  \mbox{Primal formulation:} \ & 
  \tau (\mathcal{T}(\partial \f_S(\w^t)) + \mathcal{T}(\partial \f_S(\w^t)^*) + O(p))
  + \mathcal{T}(\partial \f_S(\w^t)^*), \\
  \mbox{Dual formulation:} \ & 
  \tau (\mathcal{T}(\partial \f_S(\w^t)) + \mathcal{T}(\partial \f_S(\w^t)^*) + O(m k))
  + \mathcal{T}(\partial \f_S(\w^t)^*),
\end{align*}
where $\mathcal{T}(\partial \f_S(\w^t))$ and $\mathcal{T}(\partial
\f_S(\w^t)^*)$ denote the computational complexity of the JVP and VJP of $f_S$
respectively in a differentiable programming framework. This computational
complexity varies with the architecture considered. As an illustrative example,
if the network considered is a fully connected network with $D$ layers of
constant input-output dimensions $H$, the computational cost of a JVP/VJP is of
the order of $O(D H^2)$ elementary computations.

\paragraph{Conditioning of primal and dual formulations}
The convergence rate of a CG method on a problem of the form $\min_{\x \in
\RR^d} \frac{1}{2} \langle \x, Q \x\rangle - \langle \x, \c\rangle$ with $Q$
positive definite is theoretically given by
\[
  \|\x^\tau - \x^*\| 
  \leq 2 \left(\frac{\sqrt \kappa - 1}{\sqrt \kappa +1}\right)^\tau 
  \|\x^0 - \x^*\|_2, 
  \quad \mbox{for} \ \kappa = \frac{\lambda_{\max}(Q)}{\lambda_{\min}(Q)}
\]
where $\x^* = \argmin_{\x \in \RR^d} \frac{1}{2} \langle \x, Q \x\rangle -
\langle \x, \c\rangle$, $\kappa$ is the condition number associated to the
subproblem, and $\lambda_{\max}(Q)$, $\lambda_{\min}(Q)$ are respectively the
largest and smallest eigenvalues of $Q$. 

In our case, $\lambda_{\max} \coloneqq \lambda_{\max}(J_S^t (J_S^t)^*) =
\lambda_{\max}((J_S^t)^* J_S^t)$. If $p > m \times k$, we have that necessarily
$\lambda_{\min}((J_S^t)^* J_S^t) = 0$, while $\lambda_{\min} \coloneqq
\lambda_{\min}(J_S^t (J_S^t)^*) \geq 0$. Hence, the condition numbers associated
to the primal~\eqref{eq:cpt_cplxity_primal_ls} and
dual~\eqref{eq:cpt_cplxity_dual_ls} formulations are respectively, in the case
where $p > m \times k$, 
\begin{align*}
  \mbox{Primal:} \ \kappa_{\mathrm{primal}} 
  & = 1 + \frac{\gamma \lambda_{\max}}{m}\qquad
  \mbox{Dual:} \ \kappa_{\mathrm{dual}} 
  = \frac{\lambda_{\max} + (\gamma/m)}{\lambda_{\min} + (\gamma/m)}.
\end{align*}
The dual condition number can thus take advantage of the minimal eigenvalue
$\lambda_{\min}$ of $J_S^t (J_S^t)^*$. 
In practice, the rate of CG depends highly on the distribution of eigenvalues of
the linear operator considered~\citep{strakovs1991real,greenbaum1997iterative}.
Typically, if the eigenvalues are clustered rather than dispersed, the CG method
can converge must faster~\citep{greenbaum1997iterative}. 

\section{Experimental Details}\label{app:exp}

\paragraph{Architecture.}

The ConvNet we consider in Section~\ref{sec:exp} is defined by (i) a
convolutional layer with $32$ filters of kernel size $3\times 3$ followed by the
SiLU activation function~\citep{elfwing2018sigmoid} and an average pooling layer
of kernel size $2\times 2$, (ii) another convolutional layer with $64$ filters
of kernel size $3\times 3$ followed by the SiLU and an average pooling layer of
kernel size $2\times 2$, (iii) a dense layer of output dimension $256$, followed
by the SiLU, (iv) a final dense layer for classification of output size $10$.

\paragraph{Computing infrastructure}
The experiments have been run on TPUv2 (180 TFLOPS) and TPUv3 (420 TFLOPS)
machines with eight tensor nodes that is a single tray. 

\section{Additional experiments}
\label{app:additional_exp}

\subsection{Diagnosis experiments}
\paragraph{Sensitivity to inner iterations across batch-sizes.}

Figure~\ref{fig:max_iter_batch_size} displays the performance of the {\bf SPL}
algorithm using descent directions for a fixed stepsize $\gamma=1$ and varying
numbers of inner iterations and batch-sizes. We observe that for small
mini-batches or very large mini-batches (at 1024, none of the choice of inner
iterations led to convergence), the algorithm suffers from numerical
instabilities as it stops suddenly. Adding an Armijo-type line-search stabilizes
the algorithm across the board. Overall, a medium batch-size (here 128 or 256)
appears best performing.

\begin{figure}[t]
  \centering
  \includegraphics[width=0.7\linewidth]{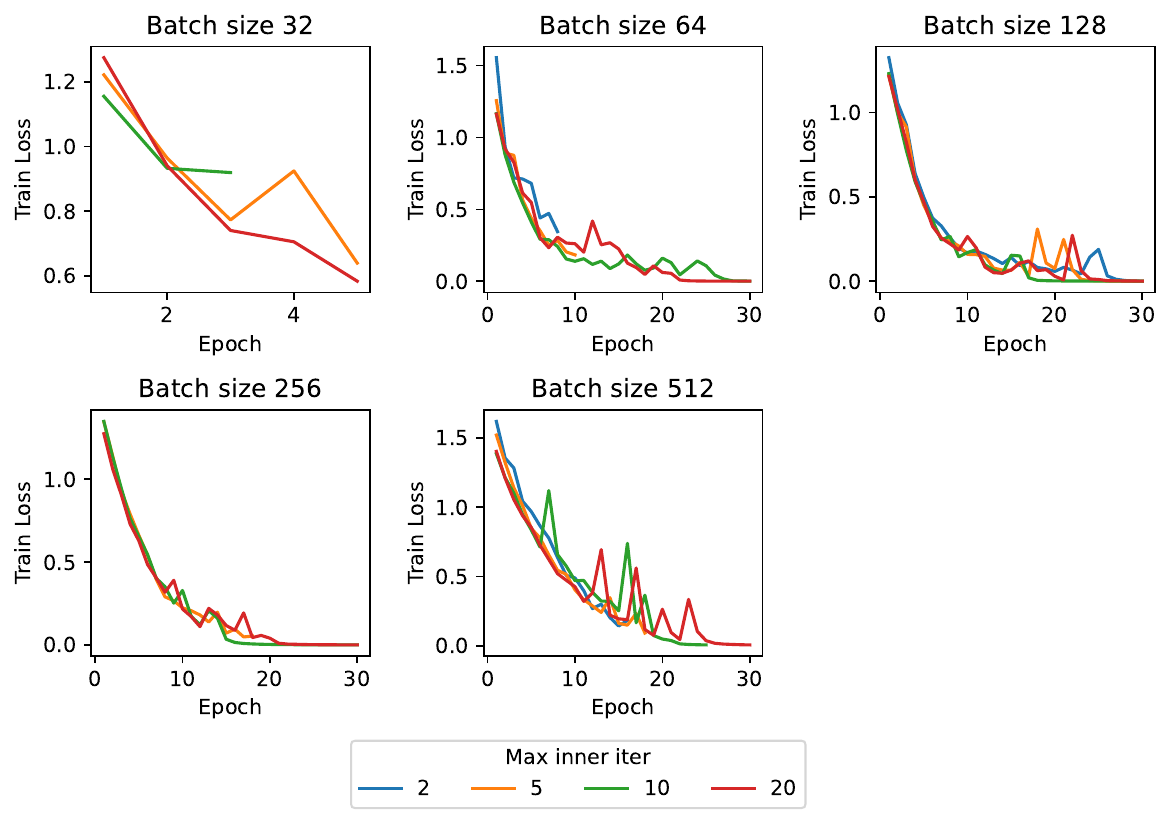}
  \caption{Sensitivity to inner iterations across batch-sizes
  for SPL without linesearch. 
  \label{fig:max_iter_batch_size}}
\end{figure}

\paragraph{Plots in time.}

Figure~\ref{fig:proxlin_x_time} presents the results of
Figure~\ref{fig:proxlin-x} in time for completeness.

\begin{figure}
  \centering
    \centering 
    \includegraphics[width=0.5\linewidth]{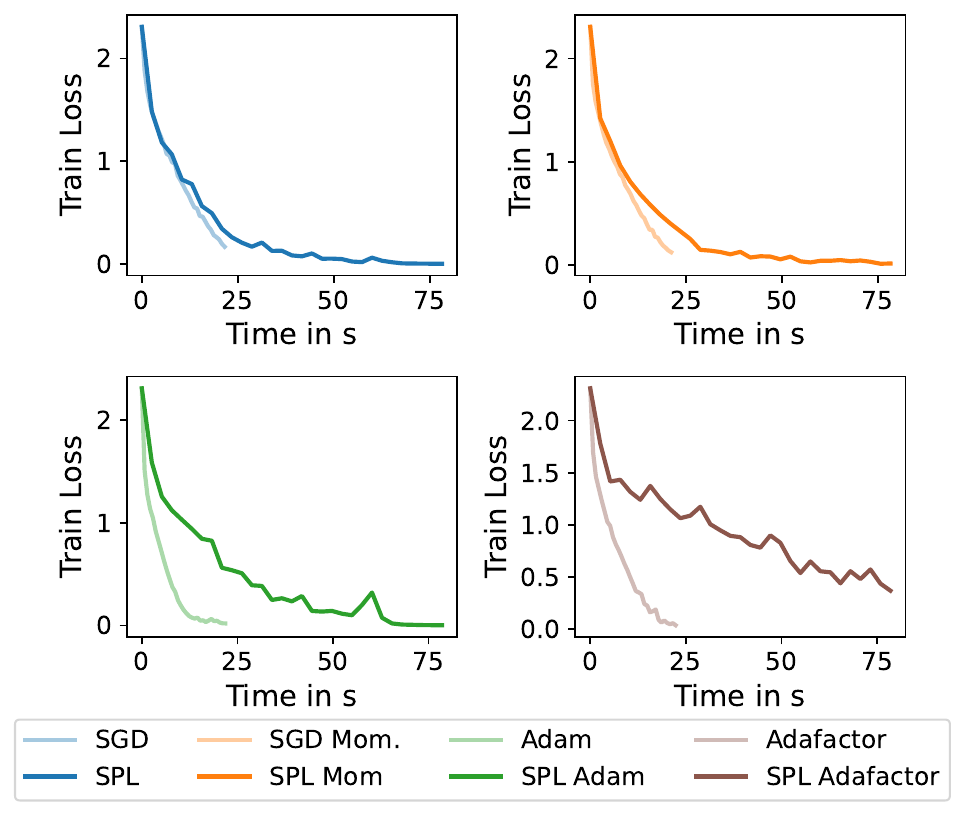}
    \caption{Prox-linear direction as a replacement for the
    stochastic gradient in existing algorithms, in time. \label{fig:proxlin_x_time}}
\end{figure}

\subsection{Additional architectures}

We consider here several other architectures ranging from CNNs
with various depths and a larger ResNet model.

\paragraph{ResNets on ImageNet.}
We consider the ResNet 18 and 34 architecture as presented by~\citet{he2016deep},
except that we consider SiLU activation functions~\citep{elfwing2018sigmoid}
instead of the ReLU activation function. We consider standard preprocessing of
images from ImageNet (crop and center images randomly to a $224\times 224$
size). 

To test the performance of the prox-linear, we consider its implementation (i)
with various $\gamma$ ranging in $\{10^i, i \in \{-2, \ldots, 2\}\}$, (ii) with
$1$ or $2$ inner iterations, (iii) with ({\bf Armijo SPL}) or without additional
linesearch ({\bf SPL}). We consider $2$ inner iterations for the inner solver.
We test SPL against SGD, SGD with momentum and Adam, whose learning rates are
searched on a log 10 scale in $\{10^i, i \in \{-5, \ldots, 0\}\}$.

The results are presented on Figure~\ref{fig:resnet}. We observe that the
potential gains offered by a Gauss-Newton method such as SPL are not apparent in
this architecture, though SPL still optimizes well the objective. 

\begin{figure}
  \includegraphics[width=\linewidth]{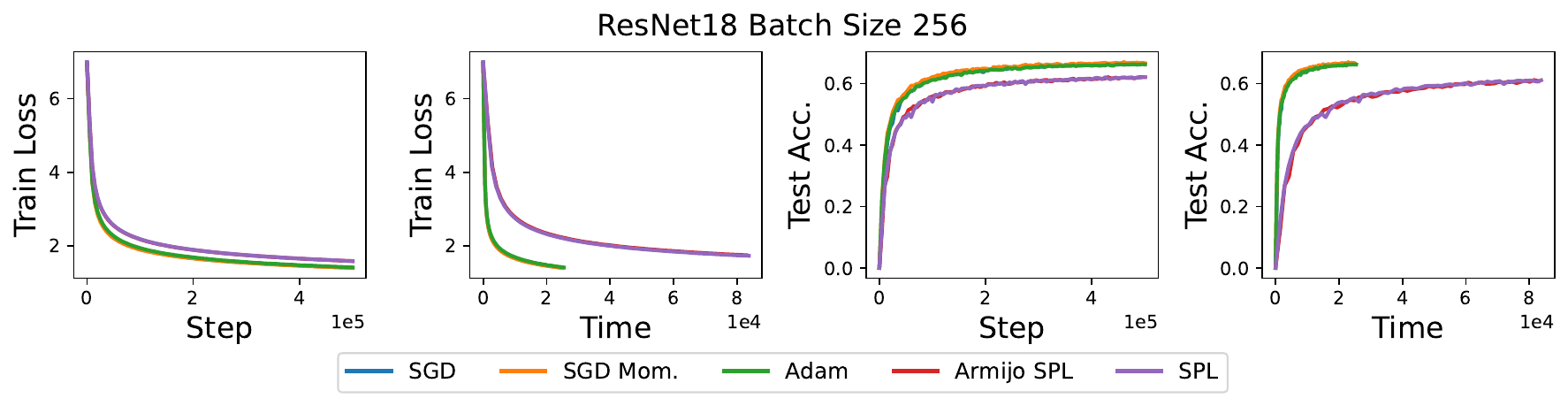}
  \includegraphics[width=\linewidth]{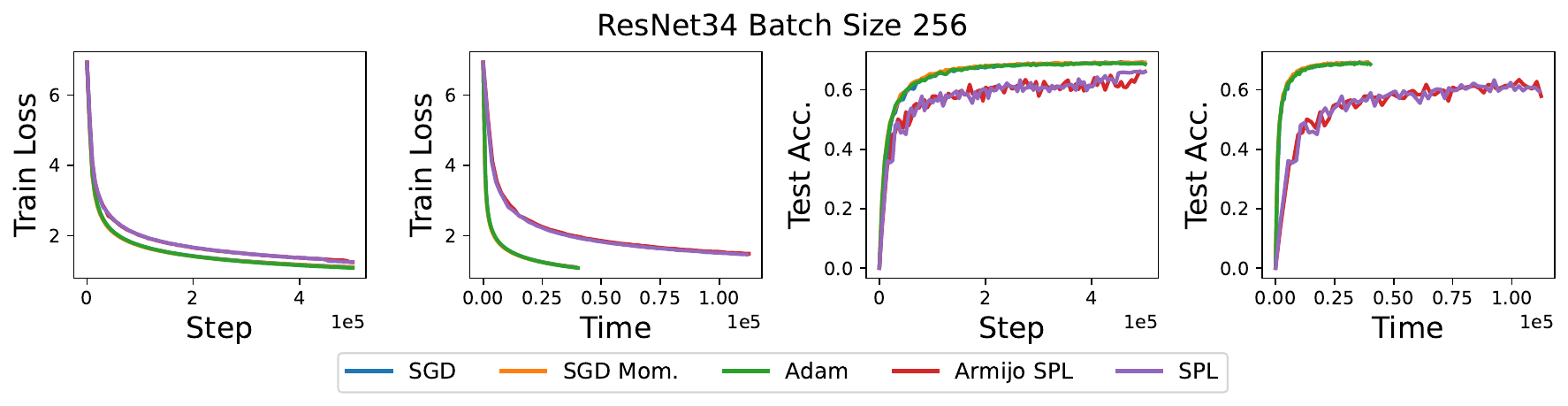}
  \includegraphics[width=\linewidth]{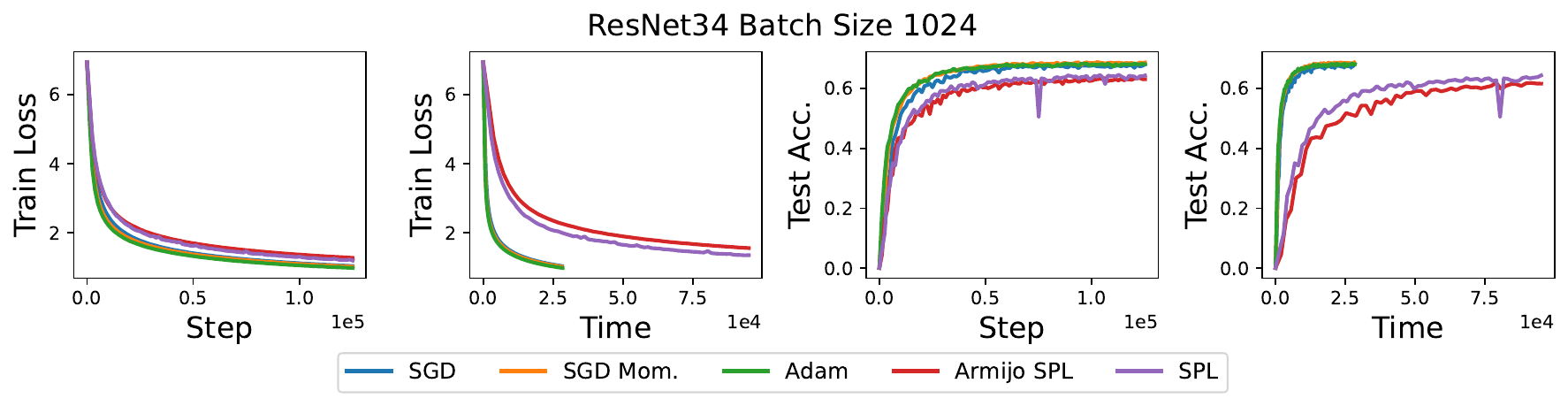}
  \caption{ImageNet.
  Top: ResNet18 batch-size 256,
  Middle: ResNet34 batch-size 256, Bottom: ResNet34 batch-size 1024, for various regularizations
  $\gamma$, various maximum inner iterations $\tau$.\label{fig:resnet}
  }
\end{figure}

\paragraph{CNNs on CIFAR10 with various depths.}
We consider additional experiments on the CIFAR10 datasets with convolutional
networks of various depths consisting in $x$ layers of $128$ filters, $x$ layers
of $64$ filters, $x$ layers of $32$ filters and a final dense layer, for $x \in
\{1, 2, 3\}$, corresponding to ``short'', ``medium'', ``long'' architectures
respectively. Each layer is followed by a SiLU activation
function~\citep{elfwing2018sigmoid}. All convolutional layers have a $3\times 3$
kernel.

In each setting, we consider an implementation of SPL (i) with various $\gamma$
ranging in $\{10^i, i \in \{-3, \ldots, 3\}\}$, (ii) with $1$, $2$, $4$ or $6$
inner iterations, (iii) with (Armijo SPL) or without additional linesearch
(SPL). We test SPL against SGD, SGD with momentum, Adam, Shampoo and KFAC, whose
learning rates are searched on a log 10 scale in $\{10^i, i \in \{-5, \ldots,
0\}\}$. For KFAC, we consider a fixed momentum of $0.9$, a fixed damping of
$10^{-3}$ and used the implementation available at {\small
\url{https://github.com/google-deepmind/kfac-jax}} with support for parallelism
on accelerators.

The results are presented on Figures~\ref{fig:short_cnn}, \ref{fig:medium_cnn},
\ref{fig:long_cnn}. We observe that SPL-Armijo performs generally on par with
its competitors, performing best for small bath-sizes. For small batch-sizes, we
also observed that the best hyper-parameters were generally to fix $\gamma$ to
$1.$, and the number of inner iterations to $2$. On the other hand, for larger
batch-sizes, we observed that the performance of Armij-SGD can deteriorate.
Interestingly, we observed that for larger models and batch-sizes the best
performances were obtained for larger number of inner iterations and smaller
$\gamma$.

\begin{figure}
  \includegraphics[width=\linewidth]{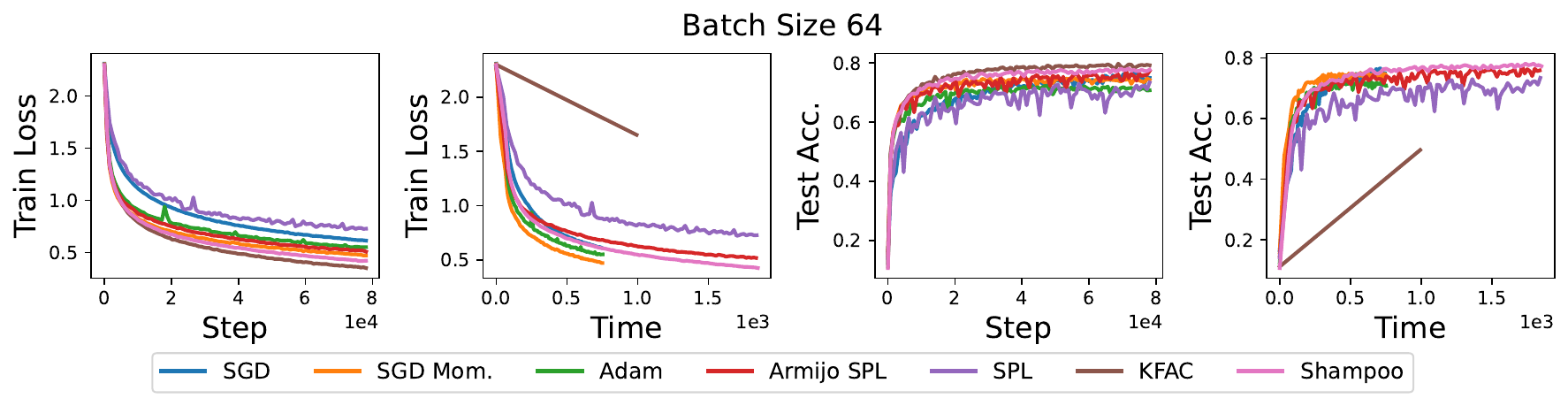}
  \includegraphics[width=\linewidth]{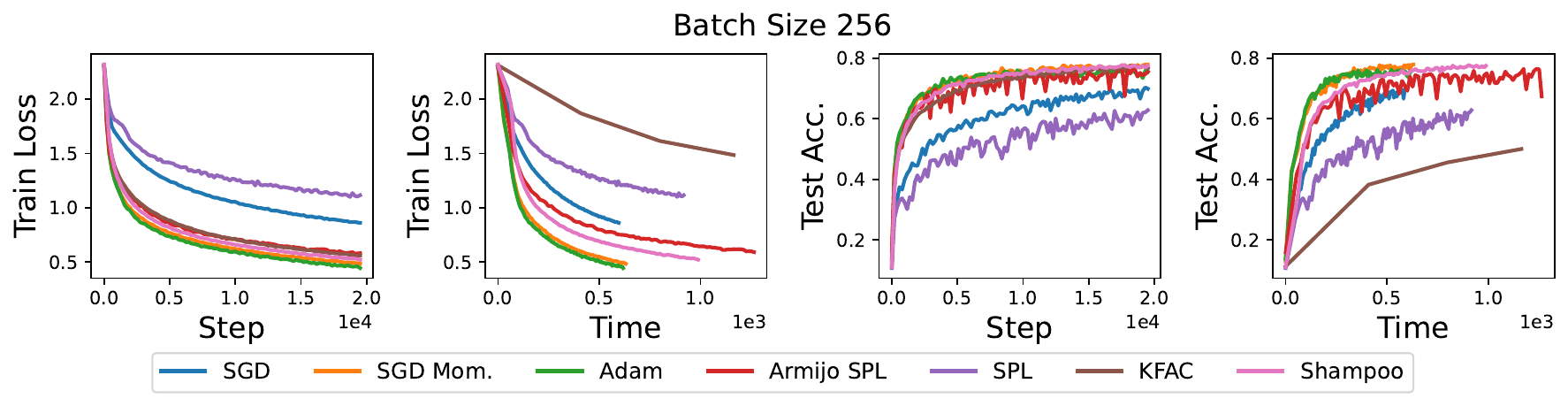}
  \caption{Short CNN $(128 \times 1, 64 \times 1, 32 \times 1)$ \label{fig:short_cnn}
  Top: Batch-size 64, Middle: Batch-size 256, Bottom: Batch-size 1024.
  Best of various regularizations $\gamma$ in $\{10^i, i \in \{-3, \ldots, 3\}\}$.
  Best of various max inner iterations $\tau$ in $(1, 2, 4, 6)$.
  }
\end{figure}

\begin{figure}
  \includegraphics[width=\linewidth]{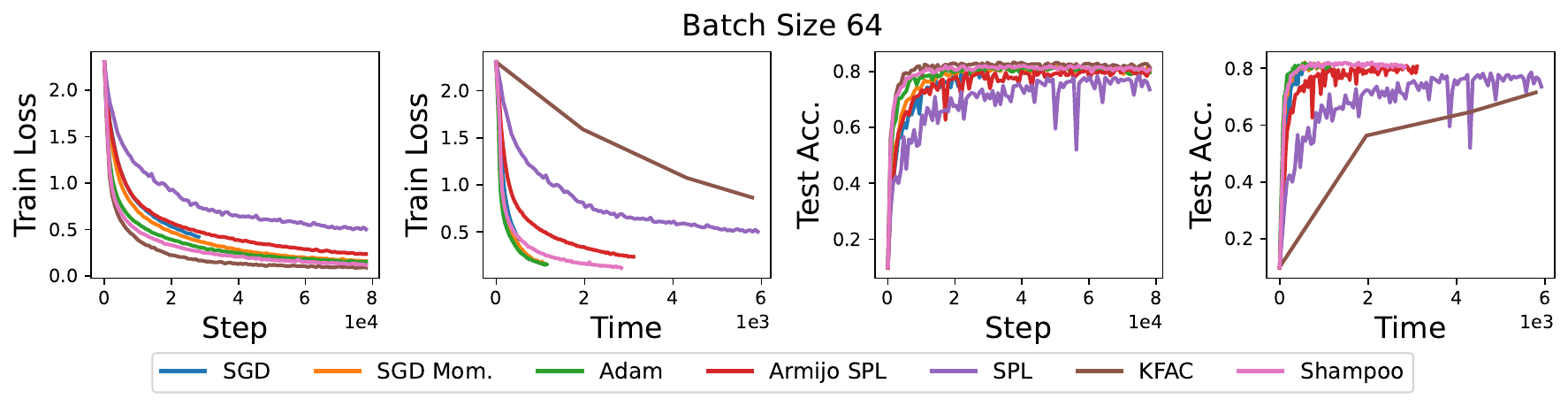}
  \includegraphics[width=\linewidth]{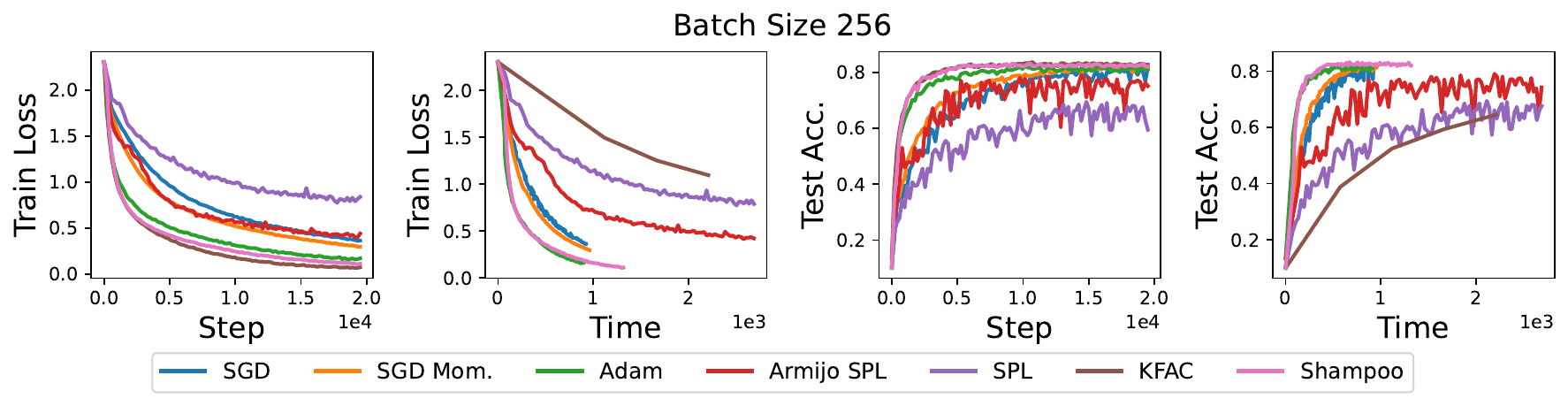}
  \caption{Medium CNN $(128 \times 2, 64 \times 2, 32 \times 2)$ \label{fig:medium_cnn}
  Top: Batch-size 64, Middle: Batch-size 256, Bottom: Batch-size 1024.
  Best of various regularizations $\gamma$ in $\{10^i, i \in \{-3, \ldots, 3\}\}$.
  Best of various max inner iterations $\tau$ in $(1, 2, 4, 6)$.
  }
\end{figure}

\begin{figure}
  \includegraphics[width=\linewidth]{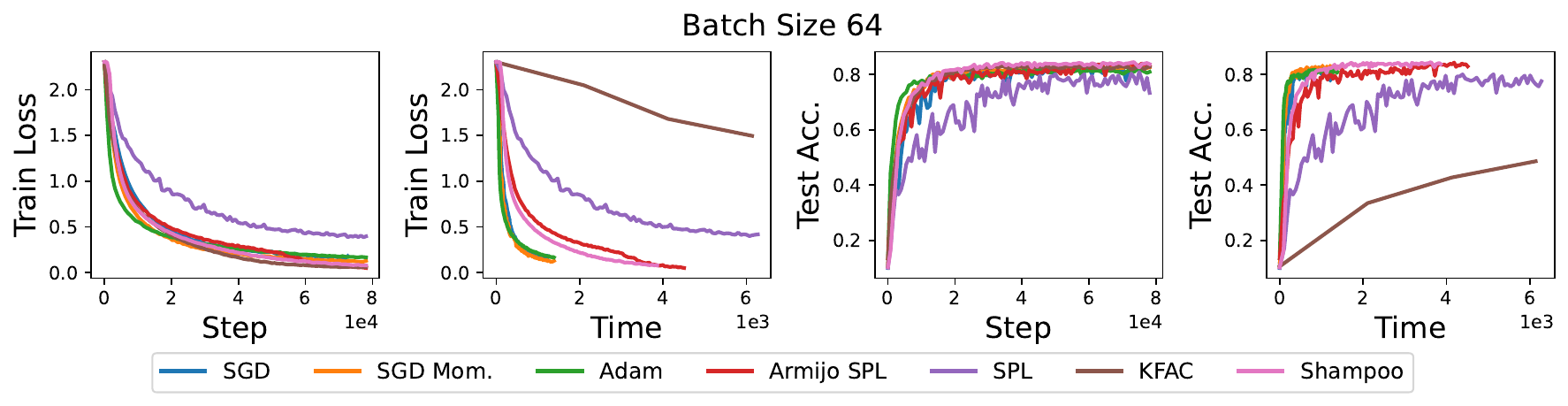}
  \includegraphics[width=\linewidth]{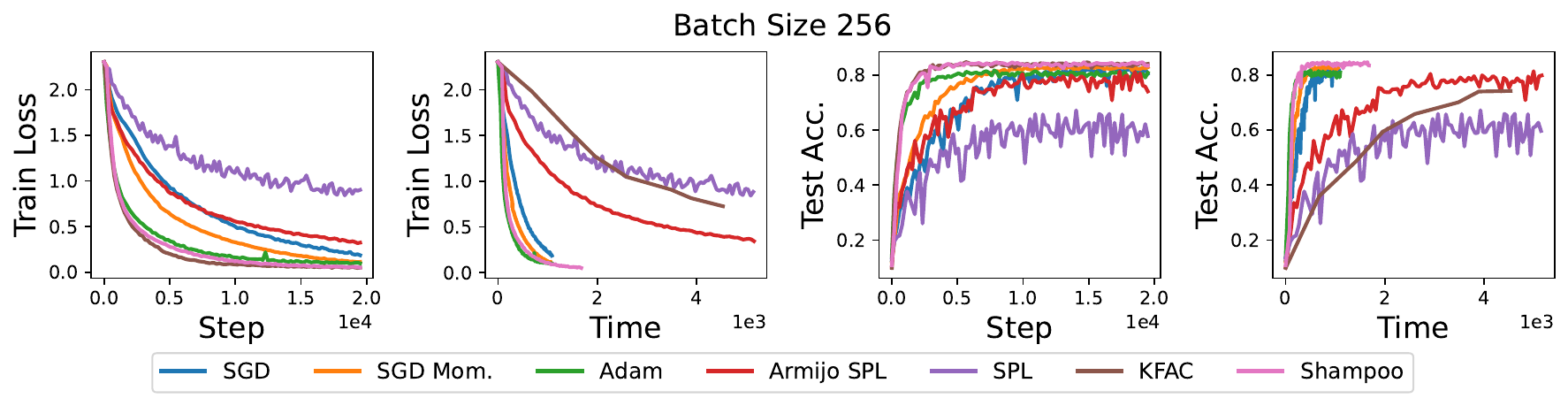}
  \caption{
  Long CNN $(128 \times 3, 64 \times 3, 32 \times 3)$. \label{fig:long_cnn}
  Top: Batch-size 64, Middle: Batch-size 256, Bottom: Batch-size 1024.
  Best of various regularizations $\gamma$ in $\{10^i, i \in \{-3, \ldots, 3\}\}$.
  Best of various max inner iterations $\tau$ in $(1, 2, 4, 6)$.
  }
\end{figure}

\end{document}